%% file: draft_kim2017semiweak.tex
\documentclass{article}

\PassOptionsToPackage{numbers}{natbib}

\usepackage{natbib}
\usepackage[margin=1in]{geometry}
\usepackage{authblk}

\usepackage[utf8]{inputenc} 
\usepackage[T1]{fontenc}    
\usepackage{hyperref}       
\usepackage{url}            
\usepackage{booktabs}       
\usepackage{amsfonts}       
\usepackage{nicefrac}       
\usepackage{microtype}      

\usepackage{times}
\usepackage{graphicx}
\usepackage{subcaption} 
\usepackage{algorithm}
\usepackage{algorithmic}
\usepackage{amsmath,amsfonts,amssymb,amsthm}
\usepackage{multirow}
\usepackage{appendix}

\theoremstyle{plain}

\newtheorem{defn}{Definition}

\newtheorem{theorem}{Theorem}
\newtheorem{prop}[theorem]{Proposition}
\newtheorem{lemma}[theorem]{Lemma}
\newtheorem{corollary}[theorem]{Corollary}

\newtheorem*{rep@theorem}{\rep@title}
\newcommand{\newreptheorem}[2]{%
	\newenvironment{rep#1}[1]{%
		\def\rep@title{#2 \ref{##1}}%
		\begin{rep@theorem}}%
		{\end{rep@theorem}}}
\newreptheorem{theorem}{Theorem}
\newreptheorem{lemma}{Lemma}
\newreptheorem{corollary}{Corollary}
\newreptheorem{prop}{Proposition}

\theoremstyle{definition}
\newtheorem{remark}{Remark}

\newcommand{\R}{\mathbb{R}}
\newcommand{\E}{\mathbb{E}}

\newcommand{\cX}{\mathcal{X}}
\newcommand{\cC}{\mathcal{C}}
\newcommand{\cS}{\mathcal{S}}
\newcommand{\bI}{\mathbb{I}}

\newcommand{\etaorg}{\frac{8(\log k+\log(m+1)+\log(1/\delta))}{(\gamma-1)^2}}
\newcommand{\etaass}{\frac{\log (2k)+\log(m+1)+\log(1/\delta)}{\log\left(1/(1-q_{assign}+q_{assign} e^{-(\gamma-1)^2/8})\right)}}
\newcommand{\etaran}{\frac{\log (2k)+\log(m+1)+\log(1/\delta)}{\log\left(1/(1-q^{k-1}+q^{k-1} e^{-(\gamma-1)^2/8})\right)}}
\newcommand{\betaran}{\frac{\log(2k)+\log\log n+\log(1/\delta)}{\log(1/(1-q))}}
\newcommand{\etadistlocal}{\frac{\log (k)+\log(m+1)+\log(1/\delta)}{\log\left(1/(1-q^{k-1}+q^{k-1} e^{-(\gamma-1)^2/8})\right)}}
\newcommand{\dimbound}{(k/\delta)^{\frac{1}{(\gamma -1)^2}-1}}

\newcommand{\iIF}[1]{\STATE\algorithmicif\ #1\ \algorithmicthen}
\newcommand{\iELSIF}[1]{\STATE\algorithmicelsif\ #1\ \algorithmicthen}
\newcommand{\iELSE}[1]{\STATE\algorithmicelse\ #1}
\newcommand{\iENDIF}{\STATE\algorithmicend\ \algorithmicif}

\title{Semi-Supervised Active Clustering with Weak Oracles}

\begin{document}

\title{\bf Semi-Supervised Active Clustering with Weak Oracles}
\author{Taewan Kim}
\author{Joydeep Ghosh}
\affil{\normalsize Department of Electrical and Computer Engineering,\\The University of Texas at Austin, USA\\
\texttt{\{twankim,jghosh\}@utexas.edu}}
\date{\vspace{-5ex}}
\maketitle

\begin{abstract}
	Semi-supervised active clustering (SSAC) utilizes the knowledge of a domain expert to cluster data points by interactively making pairwise ``same-cluster'' queries. However, it is impractical to ask human oracles to answer every pairwise query. In this paper, we study the influence of allowing ``not-sure'' answers from a weak oracle and propose algorithms to efficiently handle uncertainties. Different types of model assumptions are analyzed to cover realistic scenarios of oracle abstraction. In the first model, random-weak oracle, an oracle randomly abstains with a certain probability. We also proposed two distance-weak oracle models which simulate the case of getting confused based on the distance between two points in a pairwise query. For each weak oracle model, we show that a small query complexity is adequate for the effective $k$ means clustering with high probability. Sufficient conditions for the guarantee include a $\gamma$-margin property of the data, and an existence of a point close to each cluster center. Furthermore, we provide a sample complexity with a reduced effect of the cluster's margin and only a logarithmic dependency on the data dimension. Our results allow significantly less number of same-cluster queries if the margin of the clusters is tight, i.e. $\gamma \approx 1$. Experimental results on synthetic data show the effective performance of our approach in overcoming uncertainties.
\end{abstract}

\section{Introduction}\label{sec:intro}
Clustering is one of the most popular approaches for extracting meaningful insights from unlabeled data. However, clustering is also very challenging for a wide variety of reasons \cite{jain1999data}. Finding the optimal solution of even the simple $k$-means objective is known to be NP-hard \cite{davidson2005clustering,mahajan2009planar,vattani2009hardness,reyzin2012data}. Second, the quality of a clustering algorithm is difficult to evaluate without context.

Semi-supervised clustering is one way to overcome these problems by providing a small amount of additional knowledge related to the task. Various kinds of supervision can help unsupervised clustering: labeled samples, pairwise must-link or cannot-link constraints on elements, and split/merge requests \cite{basu2002semi,basu2004active,balcan2008clustering}. As domain experts have clear understanding about the nature of datasets, generating a small amount of supervised information should not be a challenging task for them. For example, a few pairs of samples among the large number of unlabeled animal images can be provided, and a participant can decide whether each pair must be in the same cluster or not.

Assumptions on the characteristics of a dataset itself can also assist a clustering problem. Constraints related to a margin, or a distance between different clusters, are widely used in theoretical works. Although these are strong assumptions, a margin ensures the existence of an optimal clustering, which coincides with a human expert's judgment.

The semi-supervised active clustering (SSAC) framework proposed by \citet{ashtiani2016clustering} combines both margin property and pairwise constraints in the active query setting. A domain expert can help clustering by answering same-cluster queries, which ask whether two samples belong to the same cluster or not. By using an algorithm with two phases, it was shown that the oracle's clustering can be recovered in polynomial time with high probability. However, their formulation of the same-cluster query has only two choices of answers, \textit{yes} or \textit{no}. This might be impractical as a domain expert can also encounter ambiguous situations which are difficult to respond to in a short time.

Therefore, we provide a SSAC framework that can also make a good use of weak supervision by allowing a ``not-sure'' response. We first analyze the effect of a weak oracle with a random behavior and provide the possibility of discovering the oracle's clustering with active same-cluster queries. Then several types of weak oracles are defined, and a minor assumption is shown to ensure the recovery of the oracle's clustering with high probability.

\subsection{Our Contributions}\label{subsec:contribution}
We provide novel and efficient semi-supervised active clustering algorithms for center-based clustering task, which can discover the inherent clustering of an imperfect oracle. Our work is motivated by the SSAC algorithm \cite{ashtiani2016clustering}, and the following question: ``Is it possible to perform a clustering task efficiently even with a non-ideal domain expert?''. We answer this question by formulating different types of weak oracles and prove that the SSAC algorithm can still work well under uncertainties by using properly modified binary search schemes.

The SSAC algorithm is composed of two phases, estimation of a cluster center and then of the cluster radius. Both phases are affected by not-sure answers and each phase is investigated to have a good estimation. Non-trivial strategies are developed by utilizing the characteristics of weak oracles. Our paper combines discoveries from both phases and provides a unified probabilistic guarantee for the entire algorithm's success.

Two realistic weak oracle models are introduced in the paper. First, if an oracle answers ``not-sure'' randomly with at most some fixed probability, we prove that reasonably increased sampling and query sizes can lead to a successful approximation of true cluster centers and radii of clusters. Our result generalizes the SSAC without a not-sure option in a query.

Next, we suggest practical weak-oracle model assumptions based on reasonable cases that may lead to ambiguity in answering a same-cluster query. In particular, we considered two scenarios: (i) a distance between two points from different clusters is too small, and (ii) a distance between two points within the same cluster is too large. If there exists at least one cluster element close enough to the center, an oracle's clustering can be recovered with high probability. This close point is identified from a good approximation of the cluster center and removes the uncertainty in estimating the radius of a cluster. In fact, this practical strategy is based on the idea to make use of deterministic behaviors of distance-weak oracles, and our assumption on the existence of points close to the center is very natural. Two different distance-based weak oracles are considered, and our algorithm can resolve both types.

Query complexity is obtained by utilizing a matrix concentration inequality \cite{tropp2012user}, which relies on the $\gamma$-margin property. Our new theoretical result shows that the SSAC algorithm requires less number of samples compared to the one proved by \citet{ashtiani2016clustering} when the margin between clusters is tight and the dimension of data is $O\left(\dimbound\right)$.

Finally, experimental results on synthetic data show the effective performance of our approach in overcoming uncertainties. In particular, our weak oracle model with random behavior is simulated with known ground truth, and the algorithm successfully deals with not-sure answers.

\begin{remark}
	Proofs for theoretical results are deferred to Appendix \ref{sec:append_proof} with additional analyses.
\end{remark}

\subsection{Related Work}\label{subsec:related}
Semi-supervised clustering ideas have been actively studied in 2000s \cite{basu2002semi,basu2004active,basu2004probabilistic,kulis2009semi}. \citet{basu2002semi} used seeding, or given cluster assignments on a subset of the data, as a way of supervision. Later, a similar form was considered by \citet{ashtiani2015representation}. They mapped data to a proper representation space based on the clustering of small random samples and applied $k$-means in the new space.

One of the most popular forms of supervision are pairwise constraints, i.e. must-link/cannot-link type of knowledge. \citet{basu2004active} introduced the application of these pairwise constraints in a clustering objective function and formulated based on Hidden Markov Random Fields. Then, a probabilistic framework with pairwise constraints were introduced \cite{basu2004probabilistic}, which was generalized by \citet{kulis2009semi} as a weighted kernel k-means and a graph clustering problems. Our work uses similar same-cluster queries, but interactively queries the oracle.

Active semi-supervised clustering frameworks were investigated in earlier works. \citet{cohn2003semi} proposed an iterative solution to the clustering problem using active reactions from users, but provided no theoretical guarantees on the result. \citet{basu2004active} also suggested an active semi-supervised clustering algorithm similar to our approach with an additional step of finding good pairs from the dataset, which improves the initial guess of clusters. Our result differs from their work as we consider uncertainties in queries and provide different types of probabilistic guarantees.

\citet{mazumdar2016clustering} also tackle a clustering problem with the support of oracles and side information. Distance between points in our work can be one example of side information. However, the main motivation is different from our work as we focused on not-sure answers, where they consider incorrect answers.

In this paper, we assume that the problem satisfies a center-based clustering framework. \citet{balcan2016clustering} studied algorithms to deal with perturbations on the center-based objectives including $\alpha$-center proximity. On the other hand, our work relies on the $\gamma$-margin property of data and perturbation resilience is provided as a form of high probability guarantee of success.

The most related work to this paper is \citet{ashtiani2016clustering}, which first introduced the SSAC framework. They presented the probability of recovering an oracle's clustering with an additional proof on the hardness of the problem. Instead of analyzing a NP-hardness proved by them, this paper focuses more on the performance of SSAC algorithm with weak oracles to deal with practical uncertainty issues.


\section{Problem Setting}\label{sec:problemsetup}

\subsection{Background}\label{subsec:background}
The SSAC framework was originally developed based on two important assumptions: a center-based clustering and a $\gamma$-margin property \cite{ashtiani2016clustering}. For the purpose of theoretical analysis, the domain of data is assumed to be the Euclidean space, and each center of a clustering $\cC$ is defined as a mean of elements in the corresponding cluster, i.e. $\forall i\in[k], \mu_i = \frac{1}{|C_i|}\sum_{x\in C_i}x$. Then, an optimal solution of the $k$-means clustering satisfies the conditions for center-based clustering\footnote{This holds for all Bregman divergences \cite{banerjee2005clustering}.}. 

\begin{defn}[Center-based clustering]\label{def:cb_clustering}
	Let $\cX \subset \R^m$ with $|\cX|=n$. A clustering $\cC = \{C_1,C_2,\cdots,C_k\}$ is a center-based clustering of $\cX$ with $k$ clusters, if there exists a set of centers $\mu = \{\mu_1,\cdots,\mu_k\}\subset \R^m$ satisfying the following condition: $\forall x\in\cX$ and $i \in [k], x\in C_i \Leftrightarrow i = \arg\min_{j} d(x,\mu_j),$ where $d(x,y)$ is a distance measure.
\end{defn}

Also, a $\gamma$-margin property ensures the existence of an optimal clustering. Figure \ref{fig:gamma} visually depicts the $\gamma$-margin property to help understanding the characteristic of it.

\begin{defn}[$\gamma$-margin property - Clusterability]\label{def:gamma}
	Let $\cC$ be a center-based clustering of $\cX$ with clusters $\cC =\{C_1,\cdots,C_k\}$ and corresponding centers $\{\mu_1, \cdots,\mu_k\}$. $\cC$ satisfies the $\gamma$-margin property if the following condition is true: $\forall i\in[k], \forall x\in C_i, \forall y\in \cX\setminus C_i,$
	\begin{gather*}
	\gamma d(x,\mu_i)<d(y,\mu_i)
	\end{gather*}
\end{defn}

\begin{figure}[ht]
	\begin{center}
		\centerline{\includegraphics[width=.45\linewidth]{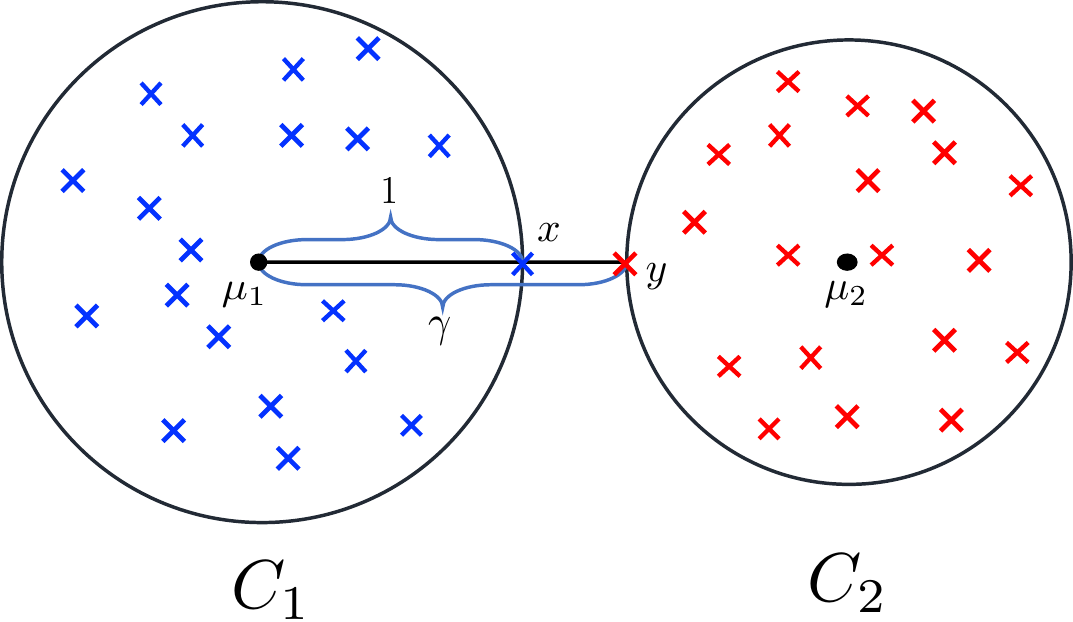}}
		\caption{Visual representation of the $\gamma$-margin property.}
		\label{fig:gamma}
	\end{center}
\end{figure}

\subsection{Problem Formulation}\label{subsec:problem}
A semi-supervised clustering algorithm is applied on data $\cX$ satisfying the $\gamma$-margin property with the oracle's clustering $\cC$, which is supported by a weak oracle that receives weak same-cluster queries.
\begin{defn}[Weak Same-cluster Query]\label{def:weak_query}
	A weak same-cluster query asks whether two data points $x_1,x_2 \in \cX$ belong to the same cluster and receives one of three responses from an oracle with a clustering $\cC$.
	\begin{align*}
	Q(x_1,x_2) = \begin{cases}
	1&\text{if } x_1, x_2\text{ are in the same cluster}\\
	0&\text{if not-sure}\\
	-1&\text{if } x_1, x_2\text{ are in different clusters}
	\end{cases}
	\end{align*}
\end{defn}

In our framework, the cluster-assignment process uses $k$ weak same-cluster queries and therefore only depends on pairwise information provided by weak oracles. In short, $k$ points with known cluster assignments from different clusters are used to determine the assignment of a given point. If an oracle outputs yes or no answer for at least $k-1$ pairwise weak queries, we can perfectly discover the cluster assignment of the point. Also, one yes answer among the $k$ weak same-cluster queries directly gives the cluster it belongs to. See Appendix \ref{subsec:append_proof_ran} for the detailed pairwise cluster-assignment process. The term ``cluster-assignment query'' is also used instead of ``weak pairwise cluster-assignment query'' throughout the paper.

\begin{defn}[Weak Pairwise Cluster-assignment Query]\label{def:weak_assign_query}
	A weak pairwise cluster-assignment query identifies the cluster index of a given data point $x$ by asking $k$ weak same-cluster queries $Q(x,y_i)$, where $y_i \in C_{\pi(i)},~i\in[k]$. One of $k+1$ responses is inferred from an oracle with $\cC=\{C_1,\cdots,C_k\}$. $\pi(\cdot)$ is a permutation defined on $[k]$ which is determined during the assignment process accordingly.
	\begin{align*}
	Q(x) = \begin{cases}
	t&\text{if } x\in C_{\pi(t)}, t\in[k]\\
	0&\text{if not-sure}\\
	\end{cases}
	\end{align*}
\end{defn}


\section{SSAC with Random-Weak Oracles}\label{sec:weak_oracle_ran}
\subsection{Random-Weak Oracle}\label{subsec:ran_weak_oracle}
One way of modeling the performance of weak oracles is to define the maximum probability of answering not-sure. We call it as a random-weak oracle, which is a natural assumption and mathematical abstraction for theoretical studies such as a binary erasure channel in information theory. This fundamental assumption is meaningful as domain experts can make mistakes or encounter hard samples with a certain frequency. In addition, some realistic scenarios can be covered by this model where there exists a chance of losing signals or not receiving answers. For example, if a restriction in time for answering a query is considered to increase the speed of an algorithm, even a perfect domain expert can miss answering a same-cluster query because of the time limit. This situation can be well depicted by the random-weak oracle model by replacing the role of a not-sure option with an event of missing an answer.

\begin{defn}[Random-Weak Oracle]\label{def:weak_oracle_ran}
	An oracle is said to be $q$ random-weak with a parameter $q\in(0,1]$, if $Q(x,y)=0$ with probability at most $1-q~$ for given two points $x,y \in \cX$.
\end{defn}

\begin{algorithm} [ht]
	\caption{SSAC for Weak Oracles}
	\label{alg:weak_SSAC_ran}
	\begin{algorithmic} [1] 
		\REQUIRE Dataset $\cX$, an oracle for weak query $Q$, target number of clusters $k$, sampling numbers $(\eta,\beta)$, and a parameter $\delta\in(0,1)$.
		\STATE $\cC=\{\},~~\cS_1 = \cX,~~r=\lceil k\eta\rceil$
		\FOR{$i=1$ to $k$}
		\STATE \vspace{.2em}\textbf{- Phase 1:}
		\STATE $Z\sim \text{Uniform}(\cS_i,r)$ \hspace{2em}// Draw $r$ samples from $S_i$
		\FOR{$1\leq t \leq k$}
		\STATE $Z_t=\{x\in Z:Q(x)=t\}$ \hspace{2em}// Pairwise cluster-assignment query
		\ENDFOR
		\STATE $p=\arg\max_t |Z_t|$, $\mu_p' \triangleq \frac{1}{|Z_p|}\sum_{x\in Z_p} x$
		\STATE \vspace{.2em}\textbf{- Phase 2:}
		\STATE $\hat{\cS_i} = \text{sorted}(\cS_i)$ \hspace{2em}// Increasing order of $d(x,\mu_p'),~x\in\cS_i$
		\STATE Select BinarySearch algorithm based on the type of a weak oracle\\
		$r_i' =$ BinarySearch($\hat{\cS_i},Z_p,\mu_p',\beta$)\hspace{2em}// Same-cluster query
		\STATE $C_p'=\{x\in \cS_i: d(x,\mu_p') < r_i' \},~~\cS_{i+1} = \cS_i \setminus C_p',~~\cC=\cC \cup \{C_p'\}$
		\ENDFOR
		\ENSURE A clustering $\cC$ of the set $\cX$
	\end{algorithmic}
\end{algorithm}

Two parts of the SSAC algorithm should be reconsidered to analyze the influence of not-sure answers from the oracle. First, number of sampled elements for cluster-assignment queries must be sufficient to accurately approximate the cluster center. Intuitively, more samples or queries are required if our semi-supervision has a chance of failure. The second step of the algorithm estimates a radius from the sample mean to recover the oracle's cluster based on distances. A binary search technique plays an important role to minimize the query complexity in logarithmic scale. However, weak oracles can lead to a situation of having failure in the intermediate \textit{search} step. Therefore, we provide a simple extension of a binary search with repetitive weak same-cluster queries, Algorithm \ref{alg:rBinary}, to mitigate the effect of uncertainties in queries. Our first main result shows the perfect recovery of the oracle's clustering on the random-weak model.
\begin{algorithm} [tb]
	\caption{Random-Weak BinarySearch}
	\label{alg:rBinary}
	\begin{algorithmic} [1] 
		\REQUIRE Sorted dataset $\hat{\cS_i}=\{x_1,\cdots,x_{|\hat{\cS_i}|}\}$ in increasing order of $d(x_j,\mu_p')$, an oracle for weak query $Q$, target cluster $p$, set of assignment-known points $Z_p$, empirical mean $\mu_p'$, and a sampling number $\beta$.
		\STATE Standard binary search algorithm with the following rules
		\STATE \vspace{.2em}\textbf{- Search}($x_j\in\hat{\cS_i}$)\textbf{:}
		\STATE Sample $\beta$ points from $Z_p$. $B\subseteq Z_p,~|B|=\beta$
		\STATE Weak same-cluster query $Q(y,x_j)$, for all $y\in B$
		\iIF{$x_j$ is in cluster $C_p$} Set left bound index as $j+1$
		\iELSIF{$x_j$ is not in Cluster $C_p$} Set right bound index as $j-1$
		\iELSIF{not-sure based on $\beta$ queries} Return fail\hspace{2em}// See Appendix \ref{sec:append_unified_binary} to handle failure
		\iENDIF
		\STATE \vspace{.2em}\textbf{- Stop:}\hspace{.5em}Found the smallest index $j^*$ such that $x_{j^*}$ is not in $C_p$
		\ENSURE $r_i'=d(x_{j^*},\mu_p')$
	\end{algorithmic}
\end{algorithm}

\begin{theorem}\label{thm:optimal_cover_ran}
	For given data $\cX$ and a distance metric $d(\cdot,\cdot)$, let $\cC$ be a center-based clustering with the $\gamma$-margin property. Let $\delta \in (0,1)$ and $\gamma>1$. If parameters $(\eta,\beta)$ for the sampling satisfy $\eta\geq \etaran$ and $\beta \geq \betaran$, then combination of Algorithm \ref{alg:weak_SSAC_ran} and \ref{alg:rBinary} outputs the oracle's clustering $\cC$ with probability at least $1-\delta$.
\end{theorem}
To prove Theorem \ref{thm:optimal_cover_ran}, we first show that a good approximate cluster center can be obtained with high probability, which leads to a simple recovery of points within the radius. Then, the probability of success in binary search steps is evaluated. Refer Appendix \ref{sec:append_proof} for detailed proofs, a query complexity, and runtime.
\begin{remark}
	The sampling number $\eta$ in theorem \ref{thm:optimal_cover_ran} generalizes the result of \citet{ashtiani2016clustering}. If queries are not weak, i.e. $q=1$, we can achieve the sampling complexity $\etaorg$, and their bound can be recovered by using a dimension independent concentration inequality. Section \ref{subsec:compare} explains the advantage of our approach.
\end{remark}
\begin{remark}
	Number of samples $(\eta,\beta)$ in Theorem \ref{thm:optimal_cover_ran} depend on $q$ and $\gamma$, both of which are unknown in real settings. Although there is no explicit way to calculate the margin $\gamma$, $q$ can be approximated with the ratio of answer from an oracle. Proper parameters for the sampling, $(\eta,\beta)$, can also be obtained through trial and error.
\end{remark}
\begin{remark}
	Since our algorithm utilizes only pairwise feedback from oracles, it subsumes a vast range of general and practical assumptions on oracles. A key motivation of our weak oracle models is uncertainty caused by obscure characteristics in a pair of samples. Therefore, even if some answers for same-cluster queries in a cluster-assignment step are not-sure, remaining answers are not necessarily determined to be not-sure. Also, some answers can provide hints to discover the cluster-assignment of a given point in practice. However, our theoretical analysis on the random-weak oracle model provides a lower bound of possible realistic situations, and more practical models for the motivation are covered in Section \ref{sec:weak_oracle_dist}.
\end{remark}

\subsection{Comparison to Dimension Independent Result}\label{subsec:compare}
A sampling number provided by \citet{ashtiani2016clustering} is $O\left(\frac{\log k+\log(1/\delta)}{(\gamma-1)^4}\right)$, which is required to guarantee a good approximation of a cluster center with high probability. This result is founded on the dimension independent concentration inequality \cite{ashtiani2015dimension}. However, $1/(\gamma-1)^4$ can be extremely large if the margin between clusters $\gamma$ is tight, i.e. $\gamma=1+\varepsilon_{\cX}$ for some small $\varepsilon_{\cX}\in(0,1)$. Our result decreased the influence of $\gamma$ by using \textit{Vector Hoeffding's Inequality} (See Theorem \ref{thm:hoeff_vector} in Appendix \ref{sec:append_concentration}) to obtain $O\left(\frac{\log k+\log m +\log(1/\delta)}{(\gamma-1)^2}\right)$ sample complexity when the oracle is not weak. In particular, if the dimension of data is $m = O\left(\dimbound\right)$, our approach ensures smaller query complexity.


\section{SSAC with Distance-Weak Oracles}\label{sec:weak_oracle_dist}
In the previous section, an oracle is assumed to have an arbitrary behavior for answering weak same-cluster queries. One advantage of such an assumption is a wide coverage over different realistic situations. However, it is more reasonable to evaluate the performance of domain experts reflecting the range of knowledge or inherent ambiguities of the given pairs of samples. The cause of not-sure answer for the same-cluster query can be investigated based on a distance between the elements in a feature space. Two cases for having indefinite answers are considered in this work: (i) points from different clusters are too close, and (ii) points within the same cluster are too far. The first situation happens a lot in the real world. For instance, distinguishing wolves from dogs is not an easy task if a data sample like a Siberian Husky is provided as visual features. The second case is also rational, because it might be difficult to compare characteristics of two points within the same cluster if they have quite dissimilar features.

\begin{algorithm} [ht]
	\caption{Distance-Weak BinarySearch}
	\label{alg:dBinary}
	\begin{algorithmic} [1] 
		\REQUIRE Sorted dataset $\hat{\cS_i}=\{x_1,\cdots,x_{|\hat{\cS_i}|}\}$ in increasing order of $d(x_j,\mu_p')$, a distance-weak oracle for weak query $Q$, target cluster $p$, set of assignment-known points $Z_p$, and empirical mean $\mu_p'$.
		\STATE Select a point $x_1$ and use it for same-cluster queries
		\STATE \vspace{.2em}\textbf{- Search}($x_j\in\hat{\cS_i}$)\textbf{:}
		\iIF{$Q(x_1,x_j)=1$} Set left bound index as $j+1$
		\iELSE Set right bound index as $j-1$ \hspace{2em}// $Q(x_1,x_j)= -1$ or not-sure
		\iENDIF
		\STATE \vspace{.2em}\textbf{- Stop:}\hspace{.5em}Found the smallest index $j^*$ such that $x_{j^*}$ is not in $C_p$
		\ENSURE $r_i'=d(x_{j^*},\mu_p')$
	\end{algorithmic}
\end{algorithm}

\subsection{Local Distance-Weak Oracle}\label{subsec:dist_weak_local}
We define the first weak-oracle model sensitive to distance, a local distance-weak oracle, in a formal way to include two vague situations described before. Condition (a) and (b) in Definition \ref{def:weak_oracle_dist_local} are mathematical expression of two depicted cases (i) and (ii) respectively. These confusing cases for local distance-weak oracle are visually depicted in Figure \ref{fig:local_dist_weak} for better explanation.

\begin{figure}[ht]
	\begin{center}
		\centerline{\includegraphics[width=.75\linewidth]{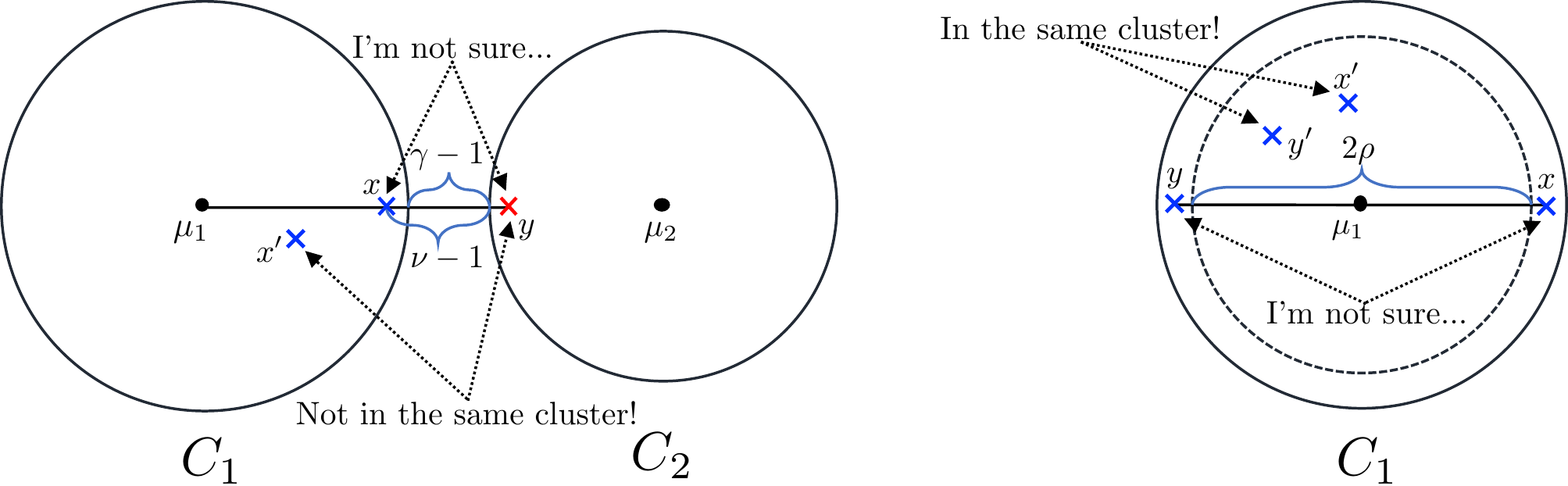}}
		\caption{Visual representation of two \textit{not-sure} cases for the local distance-weak oracle. (\textit{Left}) Two points from the different clusters are too close. (\textit{Right}) Two points from the same clusters are too far.}
		\label{fig:local_dist_weak}
	\end{center}
	\vskip -0.1in
\end{figure} 

\begin{defn}[Local Distance-Weak Oracle]\label{def:weak_oracle_dist_local}
	An oracle having a clustering $\cC=\{C_1,\cdots,C_k\}$ for data $\cX$ is said to be $(\nu,\rho)$ local distance-weak with parameters $\nu\geq 1$ and $\rho\in(0,1]$, if $Q(x,y)=0$ for any given two points $x,y \in \cX$ satisfying one of the following conditions:
	\vspace{-.5em}
	\begin{itemize}
		\setlength\itemsep{0em}
		\item[(a)] $d(x,y)<(\nu-1)\min\{d(x,\mu_i),d(y,\mu_j)\}$, where
		$x\in C_i, y\in C_j, i\neq j$
		\item[(b)] $d(x,y)>2\rho r(C_i)$, where $x,y\in C_i$
	\end{itemize}
\end{defn}

One way to overcome the local distance-weakness is to provide at least one good point in a query. If one of the points $x$ and $y$ for the query $Q(x,y)$ is close enough to the center of a cluster, a local distance-weak oracle does not get confused in answering. This situation is realistic because one representative data sample of a cluster might be a good baseline when comparing to other elements. The next theorem is founded on this intuition, and we show that a modified version of SSAC will succeed if at least one representative sample per cluster is suitable for the weak oracle. In the proof, we first show the effect of a point close to the center on weak queries. Then the possibility of having a close empirical mean is provided by defining good sets and calculating data-driven probability of failure from it. Last, an assignment-known point is identified to remove the uncertainty of same-cluster queries used in the binary search step.

\begin{theorem}\label{thm:optimal_cover_local}
	For given data $\cX$ and a distance metric $d(\cdot,\cdot)$, let $\cC$ be a center-based clustering with the $\gamma$-margin property. Let $\delta \in (0,1)$, $\gamma>1$, $\rho\in(0,1]$, $\epsilon\leq\frac{\gamma-1}{2}$, and $\gamma\leq \nu \leq \gamma+1$. If a cluster $C_i$ contains at least one point $x^*\in C_i$ satisfying $d(x^*,\mu_i)<\left(\min\{2\rho-1,\gamma-\nu+1\}-2\epsilon\right) r(C_i)$ for all $i\in[k]$, then combination of Algorithm \ref{alg:weak_SSAC_ran} and \ref{alg:dBinary} outputs the oracle's clustering $\cC$ with probability at least $1-\delta$ by asking weak same-cluster queries to a $(\nu,\rho)$ local distance-weak oracle.
\end{theorem}

\subsection{Global Distance-Weak Oracle}\label{subsec:dist_weak_global}
A global distance-weak oracle fails to answer depending on the distance of each point to its respective cluster center. In this case, both elements $x$ and $y$ should be in the covered range of an oracle if they don't belong to the same cluster. This represents an oracle that is weaker when one of points is out of its knowledge. We assume to preserve the characteristic of a distance-weakness within a same cluster, i.e. the second condition of the local distance-weak oracle.

\begin{figure}[ht]
	\begin{center}
		\centerline{\includegraphics[width=.6\linewidth]{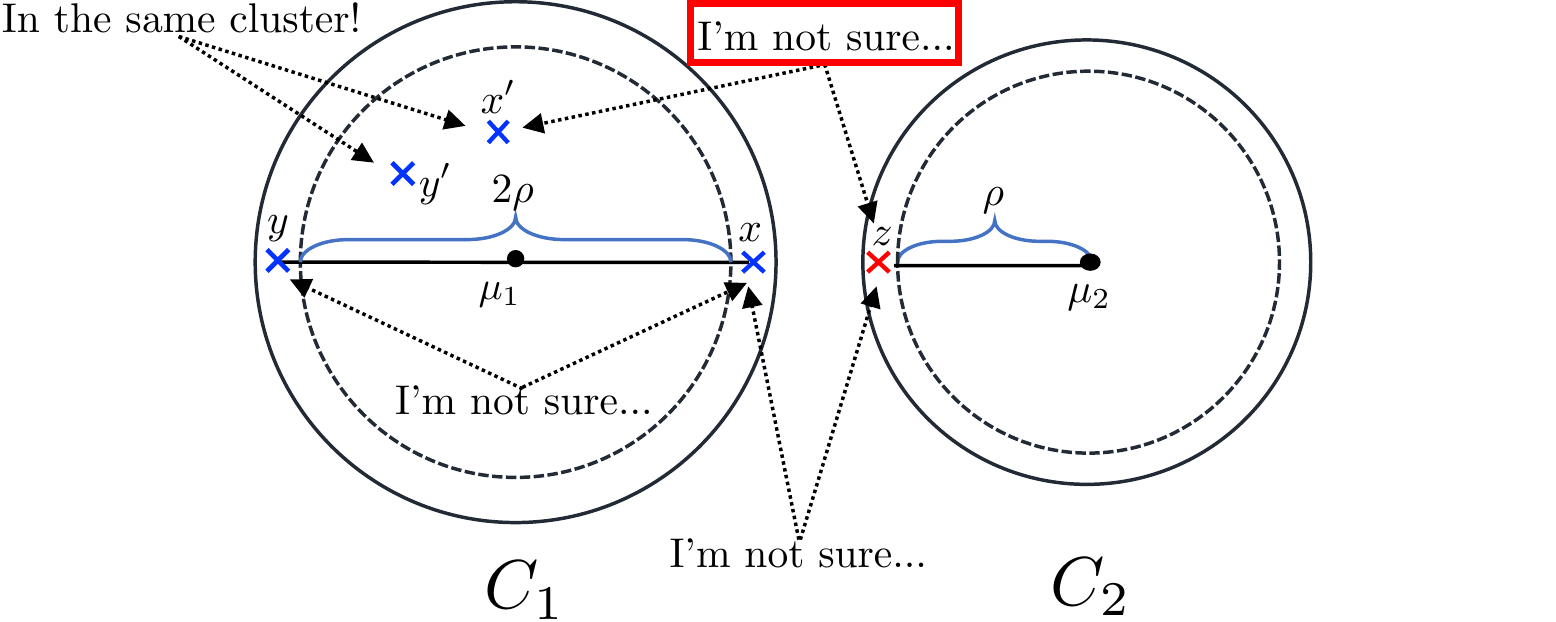}}
		\caption{Visual representation of two \textit{not-sure} cases for the global distance-weak oracle. The red box indicates the difference with the local distance-weak oracle.}
		\label{fig:global_dist_weak}
	\end{center}
	\vskip -0.1in
\end{figure} 
\begin{defn}[Global Distance-Weak Oracle]\label{def:weak_oracle_dist_global}
	An oracle having a clustering $\cC=\{C_1,\cdots,C_k\}$ for data $\cX$ is said to be $\rho$ global distance-weak with parameter $\rho\in(0,1]$, if $Q(x,y)=0$ for any given two points $x,y \in \cX$ satisfying one of the following conditions:
	\vspace{-.5em}
	\begin{itemize}
		\setlength\itemsep{0em}
		\item[(a)] $d(x,\mu_i)>\rho r(C_i)$ or $d(y,\mu_j)>\rho r(C_j),$ where
		$x\in C_i, y\in C_j, i\neq j$
		\item[(b)] $d(x,y)>2\rho r(C_i)$, where $x,y\in C_i$
	\end{itemize}
\end{defn}

The problem of a global distance-weak oracle compared to the local distance-weak model is the increased ambiguity in distinguishing elements from different clusters. Nevertheless, once we get a good estimate of the center, one good point can be still found to support same-cluster queries in the binary search step. Therefore, Algorithm \ref{alg:weak_SSAC_ran} and \ref{alg:dBinary} can guarantee the recovery of oracle's cluster with high probability by utilizing a global distance-weak oracle.

\begin{theorem}\label{thm:optimal_cover_global}
	For given data $\cX$ and a distance metric $d(\cdot,\cdot)$, let $\cC$ be a center-based clustering with the $\gamma$-margin property. Let $\delta \in (0,1)$, $\gamma>1$, $\rho\in(0,1]$, and $\epsilon\leq\frac{\gamma-1}{2}$. If a cluster $C_i$ contains at least one point $x^*\in C_i$ satisfying $d(x^*,\mu_i)<\left(2\rho-1-2\epsilon\right) r(C_i)$ for all $i\in[k]$, then combination of Algorithm \ref{alg:weak_SSAC_ran} and \ref{alg:dBinary} outputs the oracle's clustering $\cC$ with probability at least $1-\delta$, by asking weak same-cluster queries to a $\rho$ global distance-weak oracle.
\end{theorem}

\begin{remark}
	Our novel approach (to use the closest point from the estimated center) make the binary search steps avoid simple repetitive samplings. In fact, this practical strategy is based on the idea to make use of deterministic behaviors of distance-weak oracles.
\end{remark}

\begin{remark}
	Although different binary search algorithms are developed for each weak oracle model, it is possible to unify Algorithm \ref{alg:rBinary} and \ref{alg:dBinary}. First, process a same-cluster query $Q(x_1,\cdot)$ using $x_1$, the closest point from $\mu_p'$. Then, $\beta-1$ more queries can be provided to the weak oracle with additional samples from $Z_p$ if $Q(x_1,\cdot)$ gives a not-sure answer. In fact, this unified binary search algorithm strengthens the coverage of our approach because it can handle both random and distance-weak oracles at once. (See Appendix \ref{sec:append_unified_binary} for the detailed algorithm.)
\end{remark}


\section{Experimental Results}\label{sec:exp_results}
In practice, simulating active queries with a domain expert and evaluating probabilistic results is not easy as one can ``game'' the system. Therefore, simple cases on synthetic data are simulated where the true cluster assignments are known, and the oracle follows the random-weak model.\footnote{The source code is available online. \url{https://github.com/twankim/weaksemi}}

\subsection{Data Generation}\label{subsec:exp_data_gen}
For simulated dataset, points of each cluster are generated from isotropic Gaussian distribution. We assume that there exists a ground truth oracle's clustering, and the goal is to recover it where labels are partially provided via weak same-cluster queries. Various parameters are considered in generating clusters: number of samples $n$, dimension of data $m$, number of clusters $k$, and standard deviation of each Gaussian distribution $\sigma_{std}$. For visual representation, 2-dimensional data points are considered, and other parameters are set to $n=1500$, $k=3$, and $\sigma_{std}=1.75$. Data points satisfy $\gamma$-margin property with condition $\gamma_{\min}\leq \gamma \leq \gamma_{\max}$.

\subsection{Evaluation}\label{subsec:exp_eval}
Each round of the evaluation is composed of experiments with different parameter settings on $(\eta,q)$; $q$ is the probability of successful response. The unified binary search algorithm is used which handles uncertainty by regarding `not-sure' as `in different clusters'; hence $\beta$ is fixed as $10$. Parameters are varied as $q\in\{0.7,0.85,1\}$ and $\eta\in\{2,5,10,20,50\}$ in each round, and $N_{rep}=5000$ rounds are repeated. 

Two evaluation metrics are considered: $Accuracy$ is the ratio of correctly recovered data points averaged over $n$ points, and $\# Failure$ is the total number of failures occurred at cluster-assignments. The best permutation for the cluster labels is investigated based on the distances between estimated centers and true centers for the evaluation. Formal definitions of the evaluation metrics are stated below. $\bI(\cdot)$ represents the indicator function, and $z,\hat{z}$ represent true and estimated cluster labels respectively. As similar number of points are generated per cluster, a mean accuracy averaged over clusters is not considered.
\begin{align*}
Accuracy&=\frac{1}{N_{rep}}\sum_{i=1}^{N_{rep}}\sum_{j=1}^{n}\frac{\bI_{z_j=\hat{z}_j}}{n}\\
\# Failure&=\sum_{i=1}^{N_{rep}}\bI_{fail}
\end{align*}

\subsection{Results}\label{subsec:exp_results}
\begin{figure}[ht]
	\centering
	\begin{subfigure}{.38\linewidth}
		\centering
		\includegraphics[width=\linewidth]{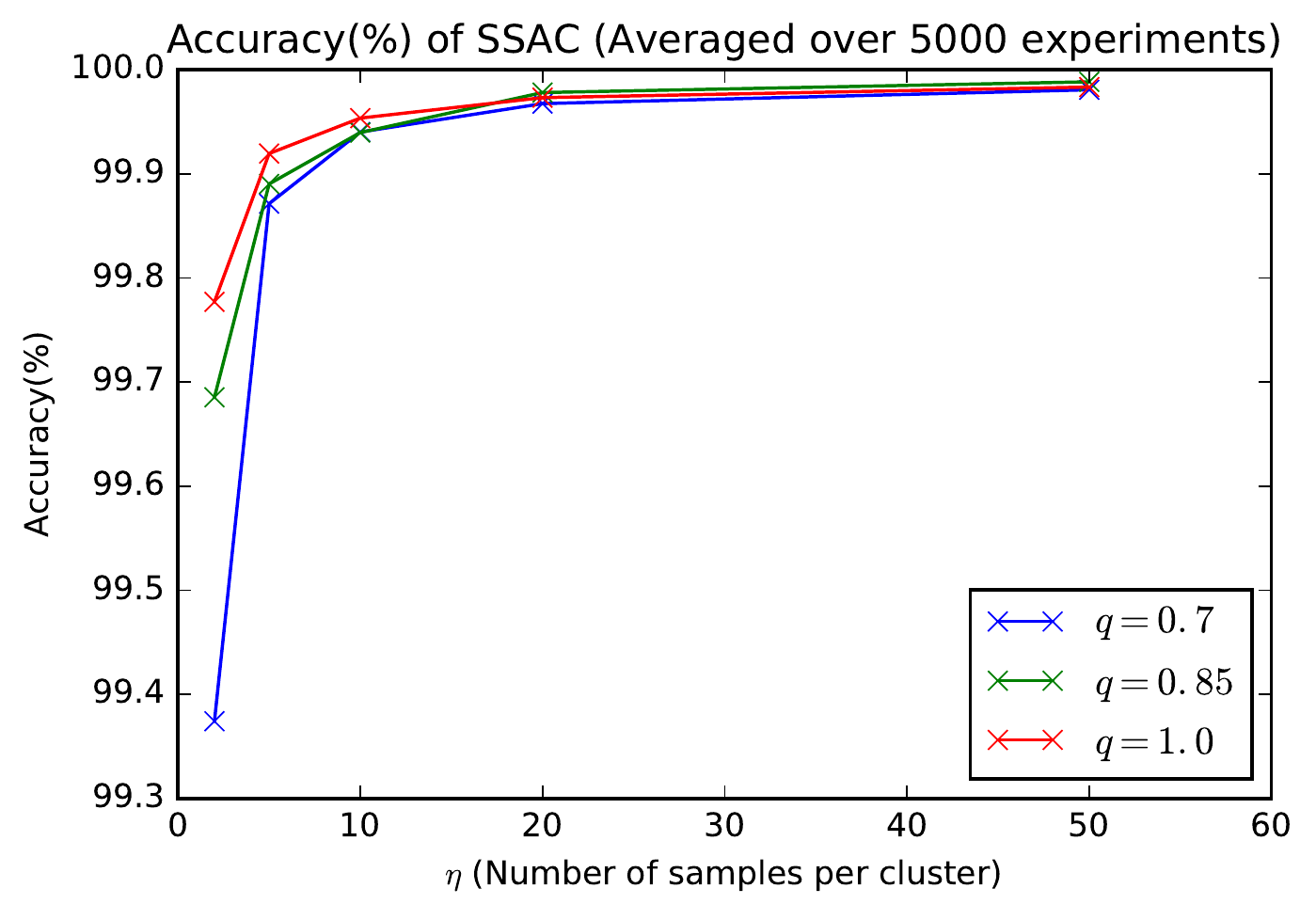}
		\caption{$Accuracy~(\%)$ of SSAC algorithm}\label{fig:res_acc_main}
	\end{subfigure}
	\hspace {2em}
	\begin{subfigure}{.38\linewidth}
		\centering
		\includegraphics[width=\linewidth]{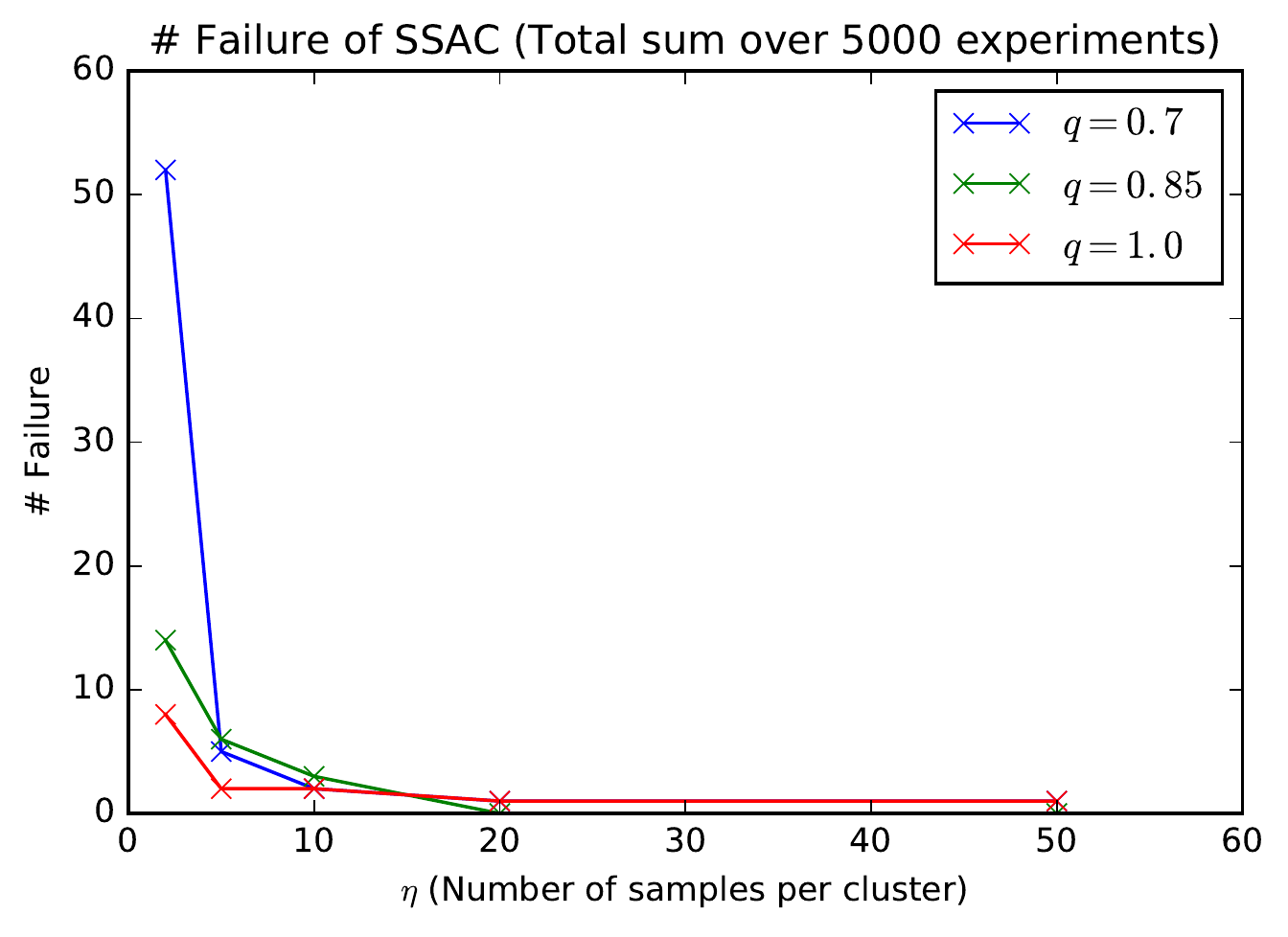}
		\caption{$\#~Failure$ of SSAC algorithm}\label{fig:res_fail}
	\end{subfigure}
	\caption{Separable case with a narrow margin. $\gamma_{\min}=1.0$, $\gamma_{\max}=1.1$. (\subref{fig:res_acc_main}) Averaged over $5000$ experiments. \textit{x-axis}: $\eta$ (Number of samples), \textit{y-axis}: $Accuracy~(\%)$. (\subref{fig:res_fail}) Total sum over $5000$ experiments. \textit{x-axis}: $\eta$, \textit{y-axis}: $\#~Failure$}
	\label{fig:res_main}
	\vskip -0.1in
\end{figure}
To focus on scenarios with narrow margins, $\gamma_{\min}=1.0$ and $\gamma_{\max}=1.1$ are chosen. Figure \ref{fig:res_main} shows $Accuracy$ in percentage and $\#~Failure$ on different parameter pairs $(q,\eta)$. An accuracy of recovering the oracle's clustering increases as $\eta$ increases. This shows the importance of enough number of samples to succeed in clustering even with uncertainties caused by an weak oracle. In fact, even small number of samples are sufficient in practice. 

Failures of the SSAC algorithm can happen as it is a probabilistic algorithm. When $\eta$ is really small, the possibility of failure increases as we have only few chances to ask cluster-assignment queries. For example, if $\eta=2$, only $r=\lceil k\eta \rceil=6$ points are sampled. Then, if all 6 cluster-assignment queries fail, Phase 1 fails. This leads to the recovery of less than $k$ clusters because the SSAC algorithm repeats Phase 1 and Phase 2 for $k$ times. However, such situations rarely occur if $\eta$ is large enough. Also, failure in binary search can happen, but we observed that only 2 out of 5000 rounds suffered from it with $\beta=10$.

\begin{figure}[ht]
	\centering
	\begin{subfigure}{.38\linewidth}
		\centering
		\includegraphics[width=\linewidth]{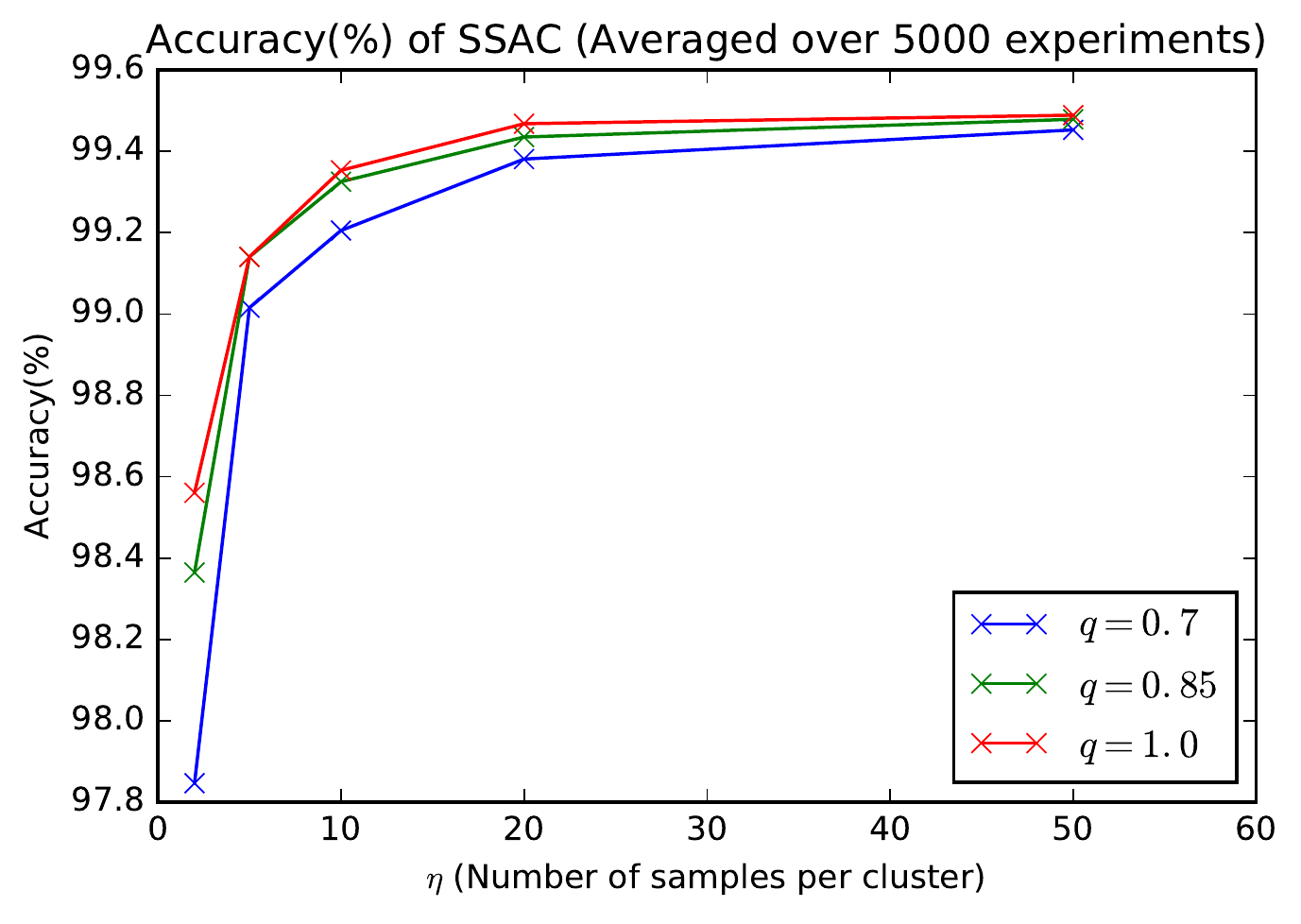}
		\caption{$Accuracy~(\%)$ of SSAC algorithm}\label{fig:res_acc_nonsep_main}
	\end{subfigure}
	\hspace {2em}
	\begin{subfigure}{.38\linewidth}
		\centering
		\includegraphics[width=\linewidth]{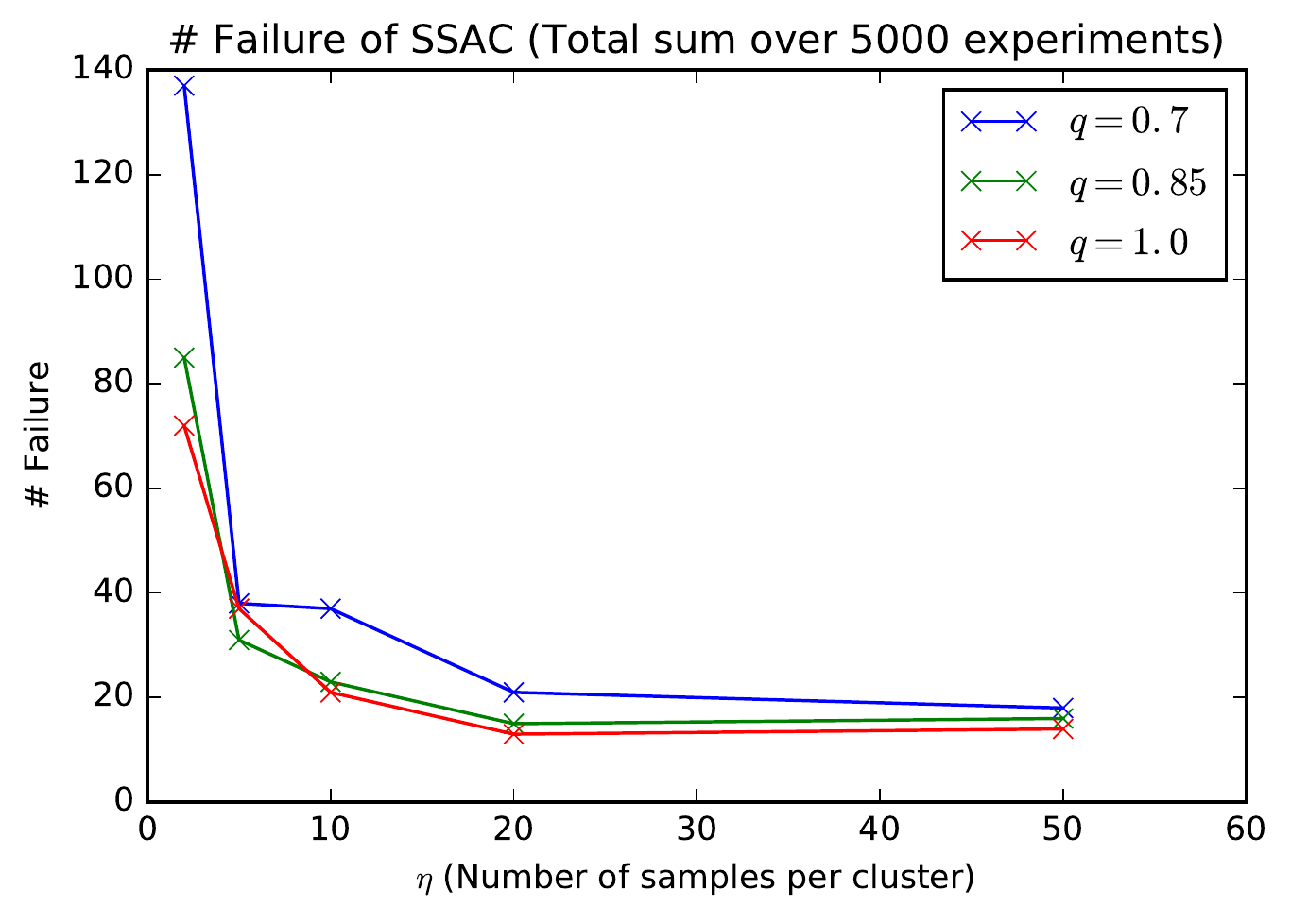}
		\caption{$\#~Failure$ of SSAC algorithm}\label{fig:res_fail_nonsep}
	\end{subfigure}
	\caption{Non-separable case. $\gamma_{\min}=0.6$, $\gamma_{\max}=1.0$. (\subref{fig:res_acc_nonsep_main}) Averaged over $5000$ experiments. \textit{x-axis}: $\eta$ (Number of samples), \textit{y-axis}: $Accuracy~(\%)$. (\subref{fig:res_fail_nonsep}) Total sum over $5000$ experiments. \textit{x-axis}: $\eta$, \textit{y-axis}: $\#~Failure$}
	\label{fig:res_nonsep_main}
	\vskip -0.1in
\end{figure}
Results on the non-separable case, $\gamma_{\min}=0.6$ and $\gamma_{\max}=1.0$, are also provided in Figure \ref{fig:res_nonsep_main}. Even if it does not get good theoretical guarantees, our algorithm still gives a reasonable clustering. See Appendix \ref{sec:append_exp_results} for additional results on different settings and scatter plots of the clusterings.


\section{Conclusion and Future Work}\label{sec:conclusion}
This paper presents approaches for utilizing a weak oracle in the semi-supervised active clustering (SSAC) framework. Specifically, we suggest two different types of domain experts that can output an answer ``not-sure'' for the same-cluster query. First, we consider a random-weak oracle that does not know the answer with at most some fixed probability. Secondly, two distance-based weak oracle models are considered to simulate realistic situations. For both of these models, probabilistic guarantees on discovering the oracle's clustering, with small dependency on the margin, are provided based on our devised binary search algorithms. In distance-based models, a single element close enough to the cluster center is able to mitigate ambiguous supervision. As our weak-oracle assumptions are designed to reflect practical scenarios, application to the real world clustering tasks with actual domain experts would be an interesting research topic. Another future direction is an extension of the framework to accommodate other distance functions or metric learning approaches.

\newpage
\bibliographystyle{plainnat}
\bibliography{kim2017semiweak}

\input{kim2017semiweak_supplemental.tex}

\end{document}

%% file: kim2017semiweak_supplemental.tex
\newpage
\clearpage
\appendix

\section{Concentration Inequality for Random Vectors}\label{sec:append_concentration}
To achieve high probability guarantees, we apply the Vector Hoeffding's inequality. Proof of Theorem \ref{thm:hoeff_vector} uses a Transpose Dilation technique on the Matrix Hoeffding results for symmetric matrices \cite{tropp2012user}.

\begin{lemma}[Matrix Hoeffding's Inequality \cite{tropp2012user}]
	\label{lemma:hoeff_matrix}
	Let $X_1, X_2,\cdots,X_s$ be i.i.d. random, symmetric matrices with dimension $m\times m$, and let $A_1,\cdots,A_s$ be fixed symmetric matrices. Assume that each random matrix satisfies,
	\begin{align*}
	\forall i \in [s],~\E[X_i]=0~\text{ and }~X_i^2 \preceq A_i^2.
	\end{align*}
	Then, for all $t\geq0$,
	\begin{gather*}
	P\left(\lambda_{\max}\left(\sum_{i=1}^{s}X_i\right) > t \right) \leq m\cdot e^{-t^2/8\sigma^2},~\text{ where }~\sigma^2 \triangleq \left\| \sum_{i=1}^s A_i^2 \right\|.
	\end{gather*}
\end{lemma}

\begin{defn}[Transpose Dilation]\label{def:dilation}
	Given a matrix $A\in \R^{d_1\ \times d_2}$, transpose dilation of $A$ is defined as a function $D_T(A): \R^{d_1 \times d_2} \rightarrow \R^{(d_1+d_2) \times (d_1 + d_2)}$:
	\begin{align*}
	D_T(A)= \begin{bmatrix}
	0 & A \\
	A^T & 0
	\end{bmatrix}
	\end{align*}
\end{defn}
Well known property about the transpose dilation is the fact that it preserves spectral information of the input matrix \cite{tropp2012user}, i.e. $\|A\|_2 = \sigma_{\max}(A) = \sigma_{\max}(D_T(A)) = \|D_T(A)\|_2=\lambda_{\max}(D_T(A))$. In short, for each singular value $\sigma_A$ of $A$, there exist two corresponding eigenvalues $+\sigma_A$ and $-\sigma_A$ of $D_T(A)$.

\begin{theorem}[Vector Hoeffding's Inequality]\label{thm:hoeff_vector}
	Let $Y_1,Y_2,\cdots,Y_s$ be $i.i.d.$ random vectors with dimension $m$, and $r_1,r_2,\cdots,r_s>0$ be a sequence of positive values. Assume that each random vector satisfies:
	\begin{align*}
	\forall i\in [s],~ \E[Y_i]=0,~ \text{ and }~\|Y_i\|_2 \leq r_i.
	\end{align*}
	Then, for any $t\geq 0$,
	\begin{align*}
	P\left(\left\|\frac{1}{s}\sum_{i=1}^{s}Y_i\right\|_2 > t \right) \leq (m+1)\cdot e^{-t^2/2\sigma^2},~\text{where }\sigma^2 \triangleq \frac{1}{s^2}\sum_{i=1}^s r_i^2
	\end{align*}
\end{theorem}

\begin{proof}
	The overall proof is motivated by dilation technique introduced by \citet{tropp2012user} so that we can apply concentration inequalities for symmetric random matrices to random vectors.
	
	Let $X_i = \frac{1}{s}D_T(Y_i)$, where $D_T(\cdot)$ is a transpose dilation defined in Definition \ref{def:dilation}. By the definition of transpose dilation, $\sum_i D_T(X_i) = D_T(\sum_i X_i)$, and $D_T(cX_i)=cD_T(X_i)$. Combining these with the fact that $D_T(\cdot)$ preserves spectral information gives,
	\begin{align*}
	\left\| \frac{1}{s}\sum_{i=1}^s Y_i \right\|_2 
	&= \left\| D_T\left(\sum_{i=1}^s\frac{1}{s}Y_i\right)\right\|_2\\
	&= \lambda_{\max} \left( D_T\left(\sum_{i=1}^s\frac{1}{s}Y_i\right)\right)\\
	&= \lambda_{\max}\left( \sum_{i=1}^s D_T\left(\frac{1}{s}Y_i\right)\right)\\ 
	&= \lambda_{\max}\left( \sum_{i=1}^s X_i \right).
	\end{align*}
	This equality indicates that $\ell_2$-norm of the sum of vectors can be transformed to the spectral norm, or the largest eigenvalue, of the sum of matrices constructed by transpose dilation.
	
	Now let's bound the square of the random matrix $X_i$.
	\begin{align*}
	X_i^2 = \begin{bmatrix}
	\frac{1}{s^2}Y_i Y_i^T&0\\
	0 &\frac{1}{s^2}\|Y_i\|_2^2
	\end{bmatrix}
	\end{align*}
	This gives 
	\begin{align*}
	\sigma_{\max}(X_i^2)= \max\left\{\sigma_{\max}\left(\frac{1}{s^2}Y_i Y_i^T\right),\sigma_{\max}\left(\frac{1}{s^2}\|Y_i\|_2^2\right)\right\}
	=\frac{1}{s^2}\|Y_i\|_2^2
	\end{align*}
	Since a random vector $Y_i$ is bounded as $\|Y_i\|_2 \leq r_i$, we can say that $X_i^2 \preceq \frac{r_i^2}{s^2} I_{m+1}$ for any $i\in[s]$, where $I_{m+1}$ represents a $(m+1) \times (m+1)$ identity matrix. Finally, we can define the following constant:
	\begin{align*}
	\sigma^2 = \left\| \sum_{i=1}^s \frac{r_i^2}{s^2}I_{m+1}\right\|_2 = \frac{1}{s^2}\sum_{i=1}^s r_i^2
	\end{align*}
	
	Therefore, by applying directly to the matrix Hoeffding's inequality, we have,
	\begin{gather*}
	P\left(\left\|\frac{1}{s}\sum_{i=1}^{s}Y_i\right\|_2 > t \right) \leq (m+1)\cdot e^{-t^2/2\sigma^2},~\text{ where }~\sigma^2 = \frac{1}{s^2}\sum_{i=1}^s r_i^2
	\end{gather*}
\end{proof}

\section{Proofs and Supplementary Analyses}\label{sec:append_proof}
In this section, proofs for theoretical results on both random-weak oracles and distance-weak oracles are provided. Also, supplementary analyses like query complexities and feasible ranges of parameters for distance-weak oracles are presented.

First, we state Lemma \ref{lemma:close_center} which assists theoretical results by introducing a characteristic of points close enough to the cluster center.
\begin{lemma}[Lemma 5 of \citet{ashtiani2016clustering}]\label{lemma:close_center}
	For given data $\cX$ and a distance metric $d(\cdot,\cdot)$, let $\cC=\{C_1,\cdots,C_k\}$ be a center-based clustering with the $\gamma$-margin property, and $\mu=\{\mu_1,\cdots,\mu_k\}$ be the set of centers (mean of each cluster) of $\cC$. Let $\mu_i'$ be a point close to the center $\mu_i$ such that $d(\mu_i',\mu_i)<\epsilon r(C_i)$, where $r(C_i)\triangleq \max_{x \in C_i}d(x,\mu_i)$. Then if $\epsilon \leq \frac{\gamma-1}{2}$ holds, points in the cluster $C_i$ are closer to $\mu'_i$ than the points of other clusters, i.e.,
	\begin{align*}
	\forall x\in C_i, \forall y \in \cX \setminus C_i,~ d(x,\mu'_i) < d(y,\mu'_i)
	\end{align*}
\end{lemma}

\subsection{Proofs for Random-Weak Oracles}\label{subsec:append_proof_ran}
\textbf{Analysis on Weak Pairwise Cluster-assignment Query}\\
A single weak pairwise cluster-assignment query is composed of $k$ weak same-cluster queries on $k$ different pairs $(x,y_i$), where $x$ is a given point and $y_i$ is an assignment-known point from each cluster $C_{\pi(i)},i\in[k]$. Therefore, if an oracle outputs yes or no answer for at least $k-1$ weak queries, we can perfectly discover the cluster assignment of $x$. This probability is lower bounded by $q^{k-1}$ as $k(1-q)q^{k-1}+q^k\geq q^{k-1}$. So, we can conclude that the probability of having not-sure answer for a given $x$ on a cluster-assignment query is at most $1-q^{k-1}$.

Also, if we only have $k'<k$ clusters defined during the process, cluster assignment of $x$ can be identified if weak-oracle gives yes or no answers for all $k'$ same-cluster queries. In detail, if one yes answer is provided among $k'$ weak queries, $x$ can be assigned to the corresponding cluster. And $k'$ no answers is handled by assigning a new cluster $k'+1$ for $x$. Then the probability of failure in identifying a cluster-assignment is at most $1-q^{k'}\leq 1-q^{k-1}$ in this case. Accordingly, we use $1-q^{k-1}$ as an upper bound for the failure probability of a cluster-assignment query to consider the worst case for further analysis on sampling complexity.
\begin{lemma}\label{lemma:distance_ran}
	For given data $\cX$ and a distance metric $d(\cdot,\cdot)$, let $\cC=\{C_1,\cdots,C_k\}$ be a center-based clustering with the $\gamma$-margin property, and $\mu=\{\mu_1,\cdots,\mu_k\}$ be the set of centers (mean of each cluster) of $\cC$. Define $Z_p, C_p, \mu_p,$ and $\mu_p'$ as in Algorithm \ref{alg:weak_SSAC_ran} with $\epsilon\leq\frac{\gamma-1}{2}$. If number of samples $s\geq|Z_p|$ including not-sure cluster assignment is at least $\etaran$, then the probability of $d(\mu_p',\mu_p)>\epsilon r(C_p)$ is at most $\frac{\delta}{2k}$, where $r(C_i)\triangleq \max_{x \in C_i}d(x,\mu_i)$.
\end{lemma}
\begin{proof}
	Let $C_p = \{x_1,x_2,\cdots,x_{n_p}\}\subset \cX$ without loss of generality ($|C_p|=n_p$). Let $\{X_i\}_{i=1}^{s}$ be i.i.d. random vectors having values $x_j\in C_p$ with probability $\frac{1}{n_p}$ for any $j\in[n_p]$. $X_i$ represents a point randomly sampled from the cluster $C_p$. Also, let $\{\xi_i\}_{i=1}^{s}$ be i.i.d. random variables having $1$ with probability $q_{assign}$ and $0$ with probability $1-q_{assign}$, which are independent of $X_i$'s. Note that $\{\xi_i\}_{i=1}^s$ indicate the cluster-assignment queries where an oracle succeeds in the assignment with probability $q_{assign}$. Then a sample mean using only assignment-known data points from $s$ samples can be represented as follows:
	\begin{equation*}
	\frac{1}{\sum_{j=1}^s \xi_j}\sum_{i=1}^s X_i\xi_i
	\end{equation*}
	
	Now, define a new random vector $Y_i=X_i-\mu_p$ for any $i \in [s]$. Then, $\E[Y_i]=0$, and its $\ell_2$ norm is bounded as $\|Y_i\|_2=\|(X_i-\mu_p)\|_2 = d(X_i,\mu_p) \leq r(C_p)$ by definition. By combining $\xi_i$ for the chance of having not-sure samples, we can achieve an upper bound of the probability of the sample mean being not close to the true mean.
	\begin{align*}
	P\left(d(\mu_p',\mu_p)>\epsilon r(C_p)\right)&=P\left(\left\|\frac{1}{\sum_j \xi_j}\sum_{i=1}^s X_i\xi_i - \mu_p\right\|_2 >\epsilon r(C_p)\right)\\
	&=P\left(\left\|\frac{1}{\sum_j \xi_j}\sum_{i=1}^s Y_i\xi_i \right\|_2 >\epsilon r(C_p)\right)\\
	&=\sum_{\ell=0}^{s}P\left(\left\|\frac{1}{\sum_j \xi_j}\sum_{i=1}^s Y_i\xi_i \right\|_2 >\epsilon r(C_p)\middle|~\sum_{j=1}^s\xi_j=\ell\right)P\left( \sum_{j=1}^s\xi_j=\ell \right)\\
	&=\sum_{\ell=1}^{s}P\left(\left\|\frac{1}{\ell}\sum_{i=1}^\ell Y_i \right\|_2 >\epsilon r(C_p)\right)P\left( \sum_{j=1}^s\xi_j=\ell \right)\\
	&\leq \sum_{\ell=1}^{s} (m+1) e^{-\ell\epsilon^2/2}\binom{s}{\ell}q_{assign}^\ell(1-q_{assign})^{s-\ell}\\
	&\leq (m+1)\left(1-q_{assign}+e^{-\epsilon^2/2}q_{assign} \right)^s
	\end{align*}
	
	The forth equality holds as $\xi_i$ and $Y_i$ are independent, and $Y_i$'s are i.i.d. random vectors. Then the first inequality can be shown by applying Theorem \ref{thm:hoeff_vector}, or Vector Hoeffding's inequality. As $\epsilon\leq\frac{\gamma-1}{2}$, we can conclude that if a number of samples $s$ from the cluster $C_p$ including not-sure ones is at least $\etaass$, then $P\left(d(\mu_p,\mu_p')>\epsilon r(C_p)\right) \leq \frac{\delta}{2k}$.
	
	Also, the last equation is a decreasing function of $q_{assign}$ and therefore can be upper bounded by replacing $q_{assign}$ with the lower bound of the cluster-assignment success probability. This concludes the proof because we showed that $q_{assign}\geq q^{k-1}$, and the sufficient number of samples including not-sure for the guarantee becomes $\etaran$.
\end{proof}

The sampling number stated in Lemma \ref{lemma:distance_ran} is a generalized version of the original same-cluster query case. If queries are not weak, i.e. $q=1$, and the target probability is $\delta/k$, you can achieve the sampling complexity $\etaorg$, which is required to have a close empirical mean using a perfect oracle.

\begin{reptheorem}{thm:optimal_cover_ran}
	For given data $\cX$ and a distance metric $d(\cdot,\cdot)$, let $\cC$ be a center-based clustering with the $\gamma$-margin property. Let $\delta \in (0,1)$ and $\gamma>1$. If parameters $(\eta,\beta)$ for the sampling satisfy $\eta\geq \etaran$ and $\beta \geq \betaran$, then combination of Algorithm \ref{alg:weak_SSAC_ran} and \ref{alg:rBinary} outputs the oracle's clustering $\cC$ with probability at least $1-\delta$.
\end{reptheorem}
\begin{proof}
	For $i\in [k]$, the phase 1 of Algorithm \ref{alg:weak_SSAC_ran} samples $r=\lceil k\eta \rceil$ points from the set $\cS_i$. Let $C_p$ be a cluster corresponds to the sample set $Z_p$. Then, at least $\eta$ number of cluster-assignment queries, including not-sure outcomes, are processed related to $C_p$ cluster by the pigeonhole principle. Let's elaborate on this claim. If we sample $\lceil k\eta \rceil$ data, there exists both cluster assignment-known points and not-sure ones. By matching not-sure data to each $Z_i$ proportional to $|Z_i|$, it can be concluded that at least $\eta$ points are sampled from one class with the chance of failure. Then Lemma \ref{lemma:close_center} and \ref{lemma:distance_ran} ensures that a sample mean $\mu_p'$ constructed by the algorithm satisfies the property $d(x,\mu_p')<d(y,\mu_p')$ for all $x\in C_p$ and $y\in \cX\setminus C_p$ with probability at least $1-\frac{\delta}{2k}$.
	
	In Phase 2, a binary search algorithm can estimate the radius $r_i'=\max_{x\in C_p}\{d(x,\mu_p')\}$ of a cluster $C_p$ from $\mu_p'$ with $O(\log|\cS_i|)$ same-cluster queries if an oracle is perfect. However, the binary search fails if at least one search step receives not-sure output from the weak oracle. The worst case probability of a failure in the search step can be calculated as $(1-q)^\beta$. By applying union bound and use $|\cS_i|\leq n$, we can conclude that Algorithm \ref{alg:rBinary} recovers the correct $r_i'$ with probability at least $1-\frac{\delta}{2k}$ if $\beta \geq \betaran$. Note that $\lim_{q\rightarrow1^-}\log\left(\frac{1}{1-q}\right)=\infty$ and condition becomes $\beta\geq 0$. This shows the generality of our result as a same-cluster query case with a perfect oracle requires only 1 query per search.
	
	By combining the above two results, we can say that the output $C_p'$ of each iteration in the Algorithm \ref{alg:weak_SSAC_ran} is a perfect recovery of $C_p$ with probability at least $1-\frac{\delta}{k}$. Again, union bound concludes the proof as iteration runs $k$ times, i.e. the SSAC algorithm with the modified binary search recovers a clustering $\cC$ of the oracle with probability at least $1-\delta$.
\end{proof}

Sufficient complexities of same-cluster queries and running time excluding queries for random-weak oracles can be calculated based on Theorem \ref{thm:optimal_cover_ran}.

\begin{corollary}\label{cor:complexity_ran}
	Let the setting be as in Theorem \ref{thm:optimal_cover_ran} and parameters $\eta$ and $\beta$ are set to be minimum sufficient values. Then the query and computational complexity for the combination of Algorithm \ref{alg:weak_SSAC_ran} and \ref{alg:rBinary} are as follows:
	\vspace{-.5em}
	\begin{itemize}
		\item[-] $q$-weak same-cluster queries:\\
		$\text{\hspace{1em}}O\left(\frac{\log k +\log\log n + \log(1/\delta)}{\log(1/(1-q))}k\log n+ k^2 \frac{\log k+\log m +\log(1/\delta)}{\log\left(1/(1-q^{k-1}+q^{k-1} e^{-(\gamma-1)^2})\right)}\right)$
		\item[-] Running time excluding queries: $O(kmn+kn\log n)$
	\end{itemize}
\end{corollary}
\begin{proof}
	For each iteration with given sampling parameters $(\eta,\beta)$, phase 1 requires $O(k\eta)$ weak same-cluster queries, and phase 2 takes $O(\beta \log n)$ queries. Also, distance calculation and sorting in phase 2 can be done in $O(mn)$ and $O(n\log n)$ respectively per each iteration.
\end{proof}

\subsection{Proofs for Distance-Weak Oracles}\label{subsec:append_proof_dist}
Before we prove the results on distance-weak oracles, additional bounds on the pairs of data points are stated. Proof of Proposition \ref{prop:bound_given} is straightforward by using the definition and the triangle inequality.

\begin{prop}\label{prop:bound_given}
	If a clustering $\cC$ of data $\cX$ satisfies the $\gamma$-margin property and has a maximum radius $r(C_i)$, the following conditions hold:
	\vspace{-.5em}
	\begin{itemize}
		\item[(a)] $d(x,y)>(\gamma-1)\max\{d(x,\mu_i),d(y,\mu_j)\}$, for all
		$x\in C_i, y\in C_j, i\neq j$
		\item[(b)] $d(x,y)\leq 2r(C_i)$,  for all $x,y\in C_i$
	\end{itemize}
\end{prop}
\begin{proof} For $x,y$ from different clusters,
	\begin{align*}
	d(x,y)&\geq d(y,\mu_i)-d(x,\mu_i)\hspace{2em}(\because \text{Triangle inquality})\\
	&>\gamma d(x,\mu_i)-d(x,\mu_i)\hspace{2em}(\because \text{Definition \ref{def:gamma}})\\
	&=(\gamma-1)d(x,\mu_i)
	\end{align*}
	Similarly, $d(x,y)>(\gamma-1)d(y,\mu_j)$, which gives (a).
	
	Also, if $x,y$ are from the same cluster $C_i$,
	\begin{align*}
	d(x,y)\leq d(x,\mu_i)+d(y,\mu_i)&\leq 2r(C_i)\hspace{2em}(\because r(C_i)=\max_{x \in C_i}d(x,\mu_i))
	\end{align*}
	which proves (b).
\end{proof}

These inequalities imply feasible ranges of parameters $\rho$ and $\nu$ for the quality of distance-weak oracles, $\rho \in (0,1]$ and $\nu\geq \gamma$. Now let's prove our main theoretical results on ditance-weak oracles.

\begin{reptheorem}{thm:optimal_cover_local}
	For given data $\cX$ and a distance metric $d(\cdot,\cdot)$, let $\cC$ be a center-based clustering with the $\gamma$-margin property. Let $\delta \in (0,1)$, $\gamma>1$, $\rho\in(0,1]$, $\epsilon\leq\frac{\gamma-1}{2}$, and $\gamma\leq \nu \leq \gamma+1$. If a cluster $C_i$ contains at least one point $x^*\in C_i$ satisfying $d(x^*,\mu_i)<\left(\min\{2\rho-1,\gamma-\nu+1\}-2\epsilon\right) r(C_i)$ for all $i\in[k]$, then combination of Algorithm \ref{alg:weak_SSAC_ran} and \ref{alg:dBinary} outputs the oracle's clustering $\cC$ with probability at least $1-\delta$ by asking weak same-cluster queries to a $(\nu,\rho)$ local distance-weak oracle.
\end{reptheorem}
\begin{proof}
	First, we show that the local distance-weak oracle always gives yes or no answer if a given weak-query includes any $x^*\in C_i$ located close enough to the center $\mu_i$. Let $x^*\in C_i$ be a data point satisfying $d(x^*,\mu_i)<\min\{2\rho-1,\gamma-\nu+1\}\cdot r(C_i)$, and an oracle is $(\nu,\rho)$ local distance-weak. If a weak query $Q(x^*,y)$ contains $y\in C_j,~i\neq j$, then,
	\begin{align*}
	d(x^*,y)&\geq d(y,\mu_i)-d(x^*,\mu_i)>\gamma r(C_i)- (\gamma-\nu+1) r(C_i)>(\nu-1)d(x^*,\mu_i)\\
	\Rightarrow d(x^*,y)&>\gamma(\nu-1)d(y,\mu_j)>(\nu-1)d(y,\mu_j)
	\end{align*}
	Moreover, if $y\in C_i$, then,
	\begin{equation*}
	d(x^*,y)\leq d(x^*,\mu_i) + d(y,\mu_i)\leq 2\rho r(C_i)
	\end{equation*}
	Therefore, two sufficient conditions for a local distance-weak oracle stated in Definition \ref{def:weak_oracle_dist_local} are violated. Hence, any weak same-cluster query including $x^*$ can be answered by the oracle without any uncertainty. Note that additional margin of $2\epsilon$ is not used at this point, which will give a higher chance of estimating a good empirical mean.
	
	Let's define $G_i(c) \subseteq C_i$ for each cluster $C_i$ as a set of data points close to center $\mu_i$,
	\begin{align*}
	G_i(c)\triangleq\{x\in C_i:d(x,\mu_i)< c\cdot r(C_i)\}
	\end{align*}
	We know that if $Q(x,y)=0$, none of the given two points $x$ and $y$ belongs to these sets $\{G_i(c)\}_{i\in [k]}$ for $c=\min\{2\rho-1,\gamma-\nu+1\}$. Let's define a probability $q_d$ as:
	\begin{equation*}
	q_d\triangleq \min_{i\in[k]}\frac{|G_i(\min\{2\rho-1,\gamma-\nu+1\})|}{|C_i|}
	\end{equation*}
	Then for a randomly sampled pair $(x,y)$ from $\cX$, the probability of having not-sure or $Q(x,y)=0$ is at most $(1-q_d)^2$. Therefore, we can use the result of Section \ref{sec:weak_oracle_ran} by substituting $q=1-(1-q_d)^2$. Especially, if a sample complexity satisfies $\eta>\etadistlocal$ with $q=1-(1-q_d)^2$ for the phase 1, Lemma \ref{lemma:distance_ran} implies that the sample mean $\mu_p'$ is a good approximation of the cluster $C_p$'s true center with probability at least $1-\frac{\delta}{k}$.
	
	However, local distance-weak oracle does not affect the binary search part as a good empirical mean is estimated before the phase 2. By Lemma \ref{lemma:close_center}, all points of cluster $C_p$ are close to $\mu_p'$. From the assumption, there exists a point $x^*\in C_i$ that is bounded as $d(x^*,\mu_i)<\left(\min\{2\rho-1,\gamma-\nu+1\}-2\epsilon\right) r(C_i)$. Now, consider the closest point $x'$ from $\mu_p'$, i.e. $d(x',\mu_p')=\min_x d(x,\mu_p')$. Then this point $x'$ is in $G_i(\min\{2\rho-1,\gamma-\nu+1\})$, which is what we want.
	\begin{align*}
	d(x',\mu_p)&\leq d(x',\mu_p')+d(\mu_p',\mu_p)\\
	&\leq d(x^*,\mu_p')+d(\mu_p',\mu_p)\hspace{2em}(\because d(x',\mu_p')=\min_x d(x,\mu_p'))\\
	&\leq d(x^*,\mu_p)+2d(\mu_p',\mu_p)\\
	&\leq \min\{2\rho-1,\gamma-\nu+1\}r(C_p)\hspace{2em}(\because d(\mu_p',\mu_p)\leq\epsilon r(C_p) \text{ by Lemma \ref{lemma:distance_ran}})
	\end{align*}
	Therefore, using the closest point $x'$ from $\mu_p'$ as in Algorithm \ref{alg:dBinary} guarantees a yes or no answer from every weak same-cluster query. This concludes the proof as remaining steps are similar to Theorem \ref{thm:optimal_cover_ran}.
\end{proof}

\begin{reptheorem}{thm:optimal_cover_global}
	For given data $\cX$ and a distance metric $d(\cdot,\cdot)$, let $\cC$ be a center-based clustering with the $\gamma$-margin property. Let $\delta \in (0,1)$, $\gamma>1$, $\rho\in(0,1]$, and $\epsilon\leq\frac{\gamma-1}{2}$. If a cluster $C_i$ contains at least one point $x^*\in C_i$ satisfying $d(x^*,\mu_i)<\left(2\rho-1-2\epsilon\right) r(C_i)$ for all $i\in[k]$, then combination of Algorithm \ref{alg:weak_SSAC_ran} and \ref{alg:dBinary} outputs the oracle's clustering $\cC$ with probability at least $1-\delta$, by asking weak same-cluster queries to a $\rho$ global distance-weak oracle.
\end{reptheorem}
\begin{proof}
	Recall the definition of $G_i(c) \subseteq C_i$ provided in the proof of Theorem \ref{thm:optimal_cover_local}. Since $2\rho-1\leq \rho$, we can define a probability $q_d'$ with $G_i(2\rho-1)$.
	\begin{equation*}
	q_d'\triangleq \min_{i\in[k]}\frac{|G_i(2\rho-1)|}{|C_i|}
	\end{equation*}
	Then $P\left(Q(x,y)=0\right)\leq 1-{q_d'}^2$ because $x,y$ should be all in the sets $\{G_i(c)\}_{i\in[k]}$. Therefore, a same result from Theorem \ref{thm:optimal_cover_local} for the phase 1 can be obtained with $q={q_d'}^2$. Note that, $d(x^*,y)\leq 2\rho r(C_i)$ if $x^*\in G_i(2\rho-1)$ and $y\in C_i$. Then, the closest point $x'$ from $\mu_p'$ is also in $G_i(2\rho-1)$, and it can still remove the uncertainty in binary search. A global distance-weak oracle will always give $Q(x',y)=1$ for $x',y\in C_p$ and we know that $x'$ and $y$ are in different clusters otherwise. This concludes the proof.
\end{proof}

\section{Unified Binary Search Algorithm for Weak Oracles}\label{sec:append_unified_binary}
Algorithm \ref{alg:unified_Binary} is a unified version of our algorithm which can deal with both random-weak oracles and distance-based weak oracles. If the dataset satisfies sufficient conditions stated in Theorem \ref{thm:optimal_cover_local}, same guarantees can be achieved as the algorithm will always succeed in same-cluster queries. Also, a random-weak oracle case can be covered with the equivalent result since we are sampling $\beta$ points from $Z_p$. Although the algorithm does not return fail, guarantees are not affected by this difference.

\begin{algorithm} [ht]
	\caption{Unified-Weak BinarySearch}
	\label{alg:unified_Binary}
	\begin{algorithmic} [1] 
		\REQUIRE Sorted dataset $\hat{\cS_i}=\{x_1,\cdots,x_{|\hat{\cS_i}|}\}$ in increasing order of $d(x_j,\mu_p')$, an oracle for weak query $Q$, target cluster $p$, set of assignment-known points $Z_p$, empirical mean $\mu_p'$, and a sampling number $\beta \leq |Z_p|$.
		\STATE Standard binary search algorithm with following rules		
		\STATE \vspace{.5em}\textbf{- Search}($x_j\in\hat{\cS_i}$)\textbf{:}
		\STATE Select the point $x_1$ and use it for same-cluster queries
		\IF{$Q(x_1,x_j)=1$}
		\STATE Set left bound index as $j+1$
		\ELSIF{$Q(x_1,x_j)=-1$}
		\STATE Set right bound index as $j-1$
		\ELSE
		\STATE Sample $\beta-1$ points from $Z_p$. $B\subseteq Z_p,~|B|=\beta-1$
		\STATE Weak same-cluster query $Q(x_j,y)$, for all $y\in B$
		\IF{$x_j$ is in cluster $C_p$}
		\STATE Set left bound index as $j+1$
		\ELSE
		\STATE Set right bound index as $j-1$
		\ENDIF
		\ENDIF
		\STATE \vspace{.5em}\textbf{- Stop:}\hspace{.5em}Found the smallest index $j^*$ such that $x_{j^*}$ is not in $C_p$
		\ENSURE $r_i'=d(x_{j^*},\mu_p')$		
	\end{algorithmic}
\end{algorithm}

For the global distance-weak oracle, similar high probability result can be achieved, but $\beta$ times more query complexity is required in the binary search. This comes from the first \textit{else} clause in the Algorithm \ref{alg:unified_Binary}, as we cannot avoid processing all $\beta$ queries if we get a not-sure answer from $Q(x_1,x_j)$.

\newpage
\section{Additional Experimental Results}\label{sec:append_exp_results}
In this section, visualizations of the clustering results in Section \ref{sec:exp_results}, and additional experimental results are provided. First, experimental results in Section \ref{sec:exp_results} with $\gamma_{\min}=1.0$ and $\gamma_{\max}=1.1$ are provided in Table \ref{table:results}. Also, Figure \ref{fig:res_acc} depicts $Accuracy$ as a graph, and a histogram of $\gamma$ values generated in $N_{rep}=5000$ rounds can be seen in Figure \ref{fig:gamma_hist}.

\begin{figure}[ht]
	\centering
	\begin{subfigure}{.4\linewidth}
		\centering
		\includegraphics[width=\linewidth]{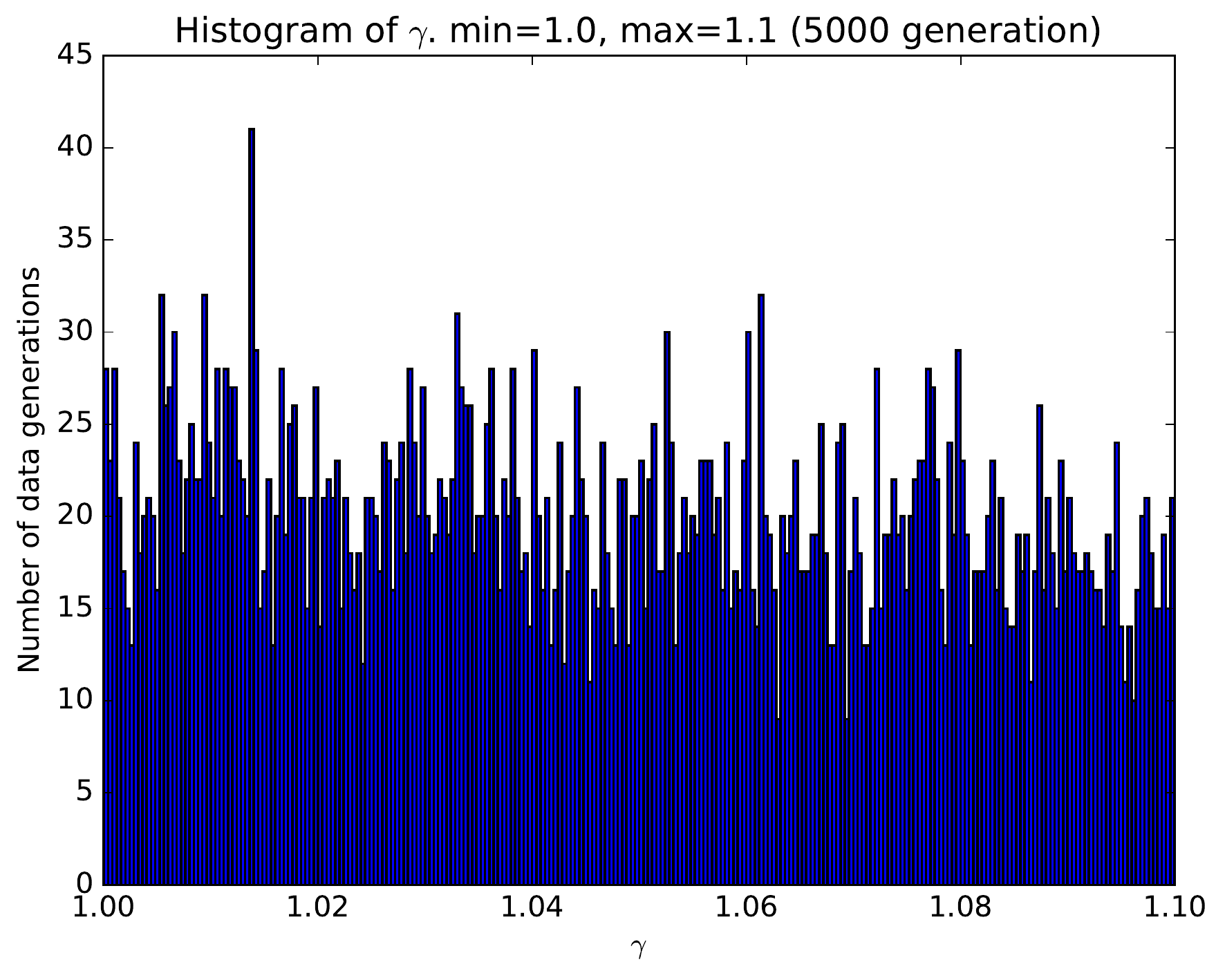}
		\caption{}\label{fig:gamma_hist}
	\end{subfigure}
	\hspace {2em}
	\begin{subfigure}{.4\linewidth}
		\includegraphics[width=\linewidth]{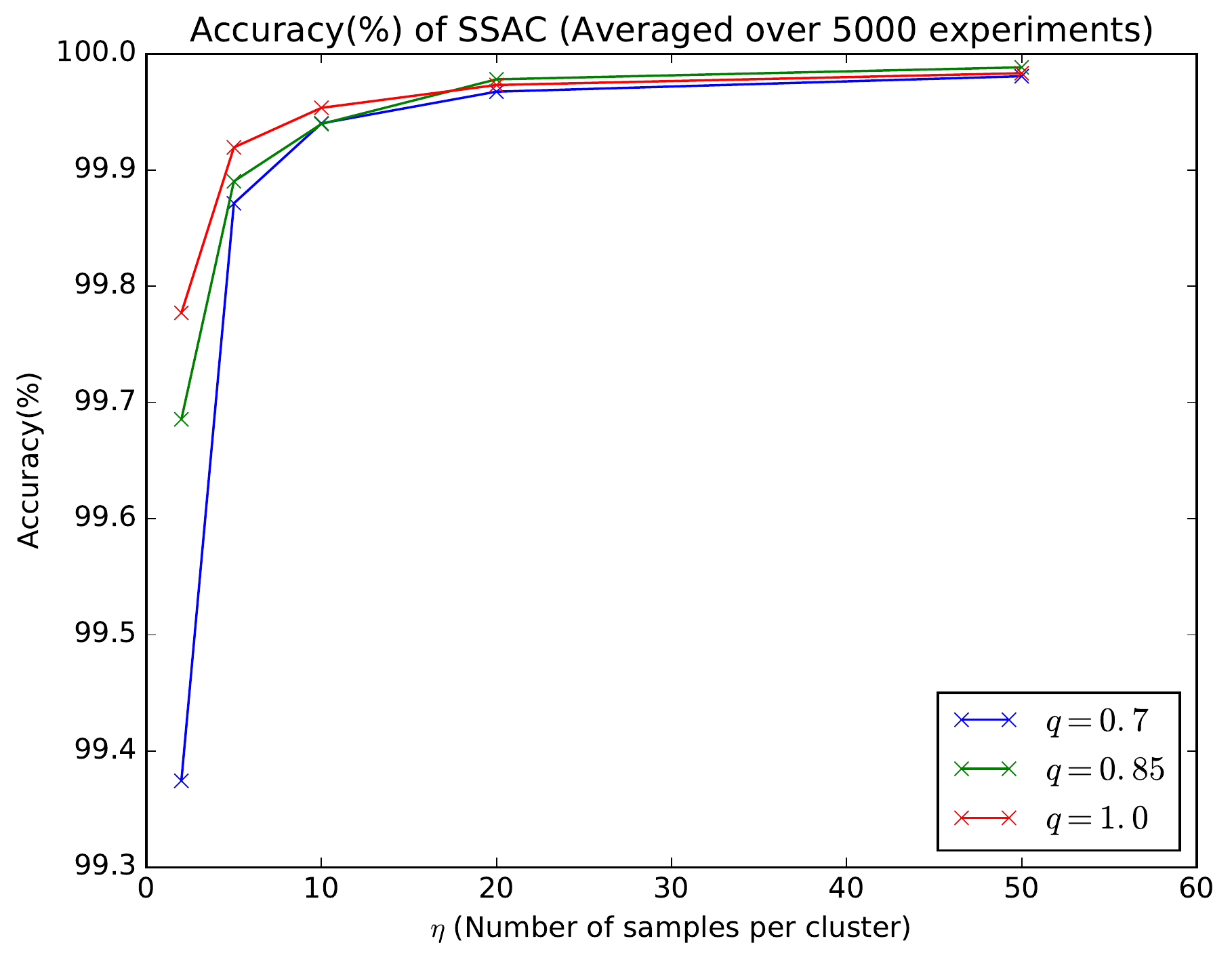}
		\caption{}\label{fig:res_acc}
	\end{subfigure}
	\caption{$\gamma_{\min}=1.0$, $\gamma_{\max}=1.1$, $m=2$. (\subref{fig:gamma_hist}) Histogram of $\gamma$ margins ($N_{rep}=5000$). \textit{x-axis}: $\gamma$ (Margin value), \textit{y-axis}: Normalized number of data generations corresponding to each $\gamma$. (\subref{fig:res_acc}) $Accuracy~(\%)$ of SSAC algorithm. Averaged over $N_{rep}=5000$ experiments. \textit{x-axis}: $\eta$ (Number of samples), \textit{y-axis}: $Accuracy~(\%)$.}
	\label{fig:res_append}
	\vskip -0.1in
\end{figure}

\begin{table}[ht]
	\begin{small}
		\centering
		\caption{$\gamma_{\min}=1.0$, $\gamma_{\max}=1.1$, $m=2$. (\textit{Left}) $Accuracy~(\%)$ of SSAC algorithm. Averaged over $N_{rep}=5000$ experiments. (\textit{Right}) $\# Failure$ of SSAC algorithm. Total sum over $N_{rep}=5000$ experiments.}
		\label{table:results}
		\vspace{1em}
		\begin{subtable}{.58\linewidth}
			\centering
			\begin{tabular}{|c|ccccc|}
				\hline
				\multirow{2}{*}{$q$}& \multicolumn{5}{c|}{$\eta$} \\
				& 2 & 5 & 10 & 20 & 50 \\
				\hline
				\hline
				0.70 & 99.374 & 99.871 & 99.940 & 99.967 & 99.981 \\
				0.85 & 99.685 & 99.890 & 99.940 & 99.978 & 99.988 \\
				1.00 & 99.777 & 99.919 & 99.953 & 99.973 & 99.983 \\
				\hline
			\end{tabular}
		\end{subtable}
		\begin{subtable}{.38\linewidth}
			\centering
			\begin{tabular}{|c|ccccc|}
				\hline
				\multirow{2}{*}{$q$}& \multicolumn{5}{c|}{$\eta$} \\
				& 2 & 5 & 10 & 20 & 50 \\
				\hline
				\hline
				0.70 & 52 &  5 & 2 & 1 & 1 \\
				0.85 & 14 &  6 & 3 & 0 & 0 \\
				1.00 &  8 &  2 & 2 & 1 & 1 \\
				\hline
			\end{tabular}
		\end{subtable}
	\end{small}
\end{table}

\subsection{Visualization of clustering results}\label{subsec:append_visualization}
Figure \ref{fig:res_q07}, \ref{fig:res_q85}, and \ref{fig:res_q10} visualize both the ground truth clustering of generated points and the recovered clustering by our algorithm. Also, cluster centers estimated in Phase 1 of SSAC algorithm are marked as white triangles. Each figure includes 5 subfigures to provide results on different values of $\eta$. Some points are assigned as cluster $0$, which represents the case where labels are not identified. This issue can happen when a cluster-assignment query fails in Phase 1, or an estimated radius cannot cover the whole points of a cluster due to the bad estimation of the center. Also, not-sure answers from same-cluster queries in Phase 2 can result in a shorter estimated radius.

\newpage
\begin{figure}[ht]
	\begin{center}
		\centerline{\includegraphics[width=.48\linewidth]{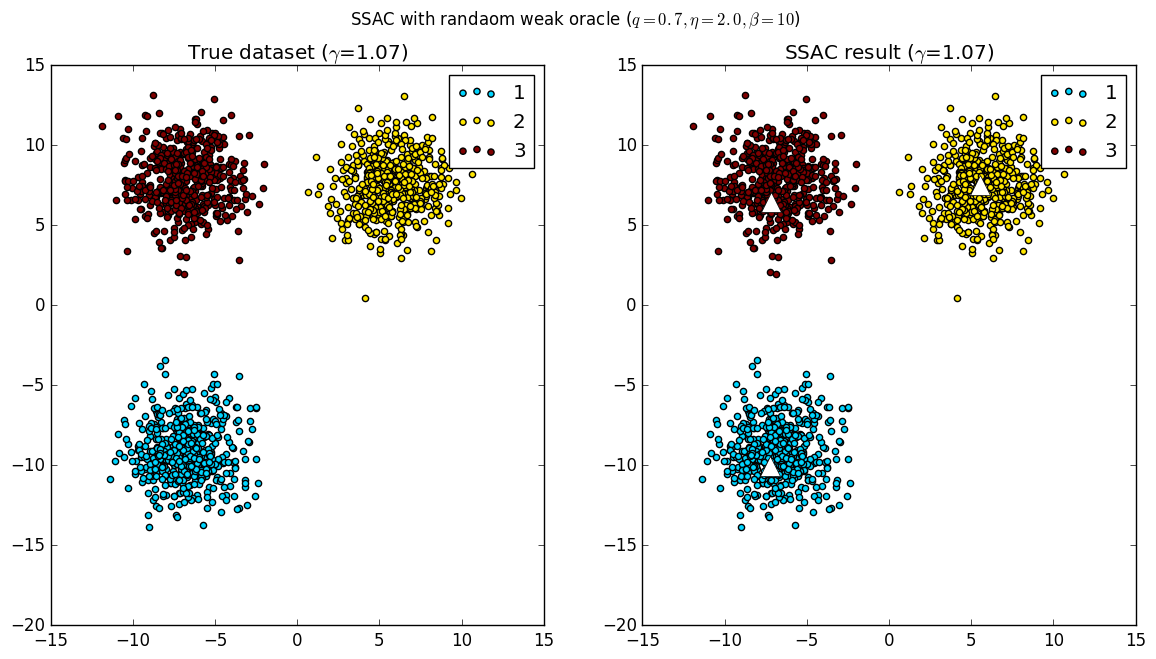}\hspace{1em}\includegraphics[width=.48\linewidth]{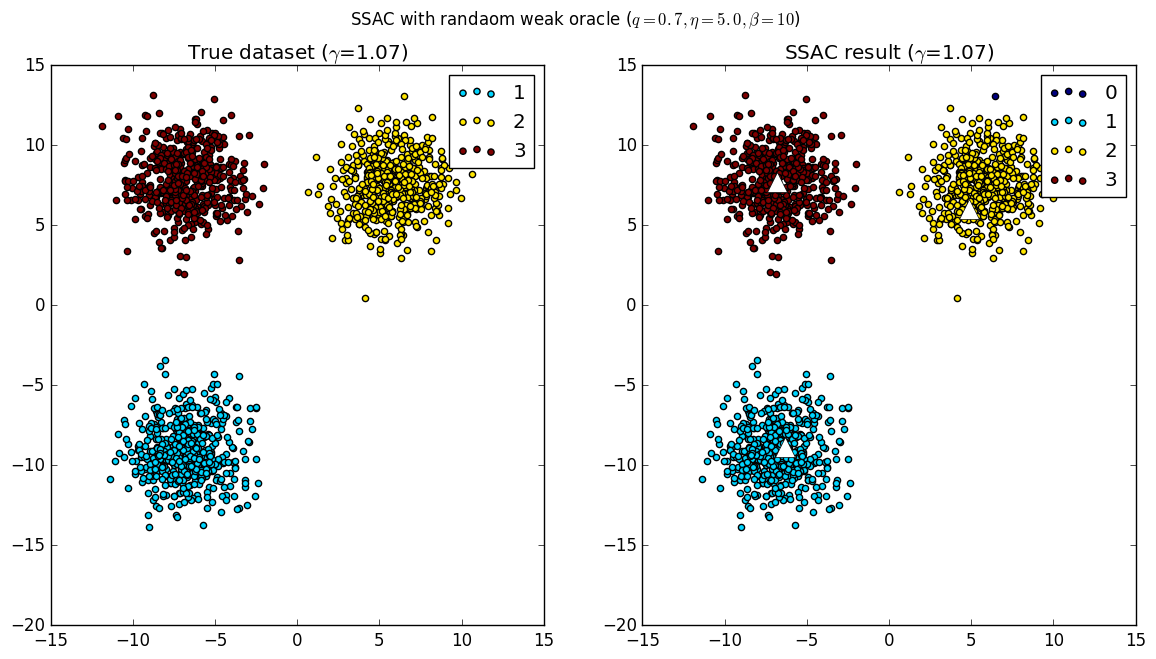}}
		\vspace{2em}
		\centerline{\includegraphics[width=.48\linewidth]{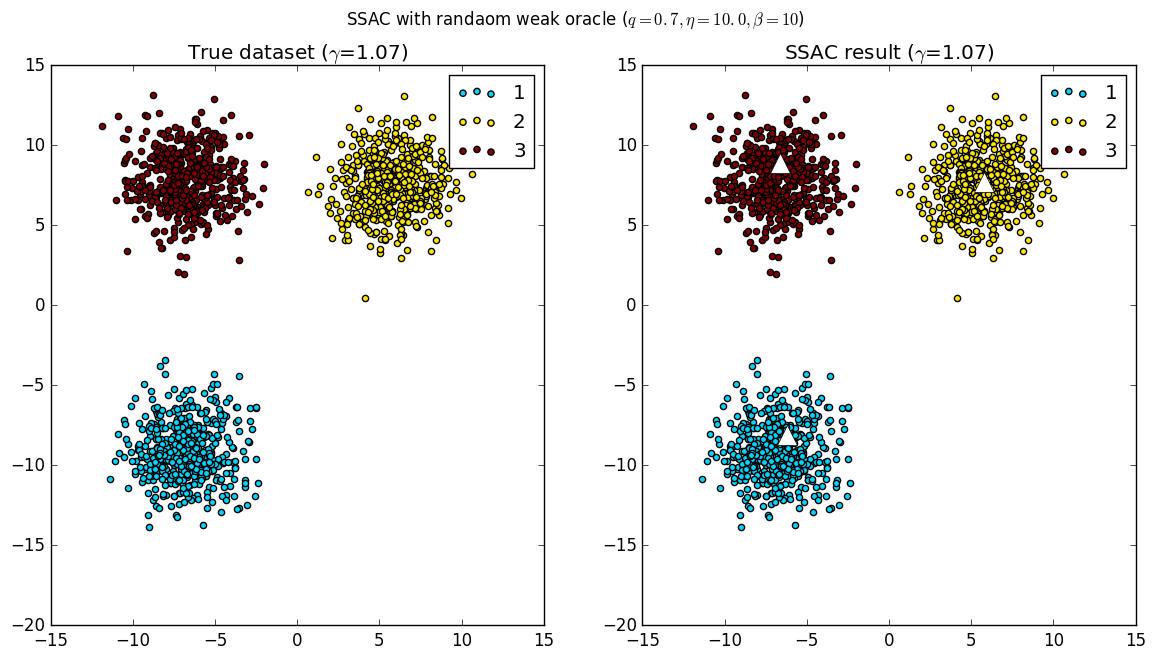}\hspace{1em}\includegraphics[width=.48\linewidth]{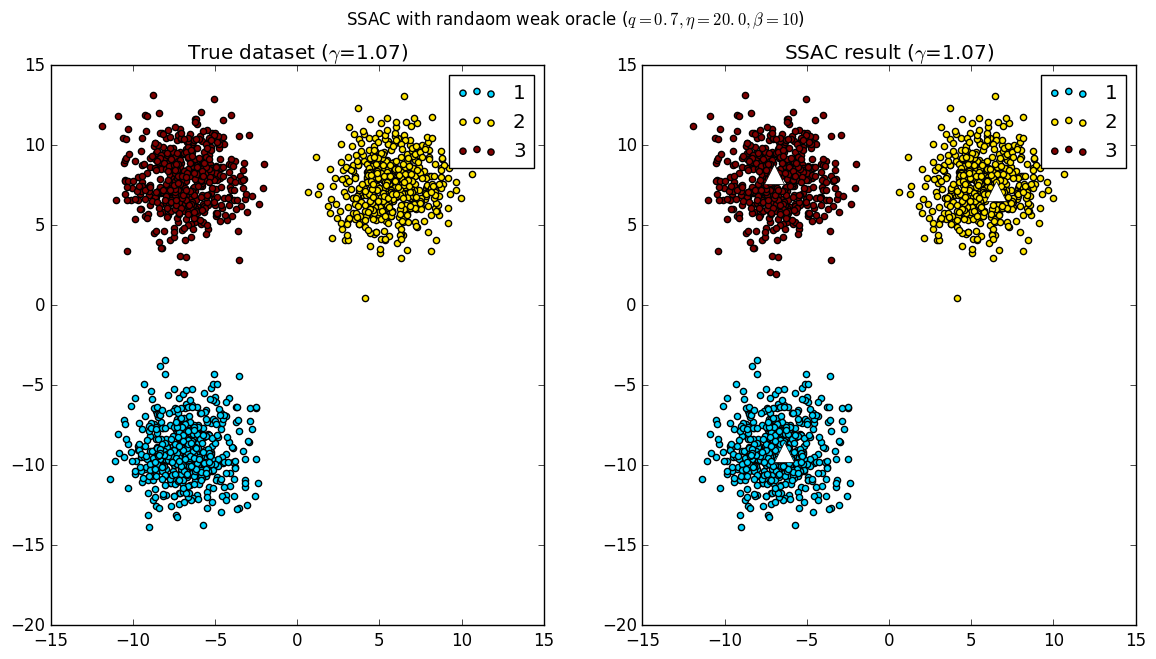}}
		\vspace{2em}
		\centerline{\includegraphics[width=.48\linewidth]{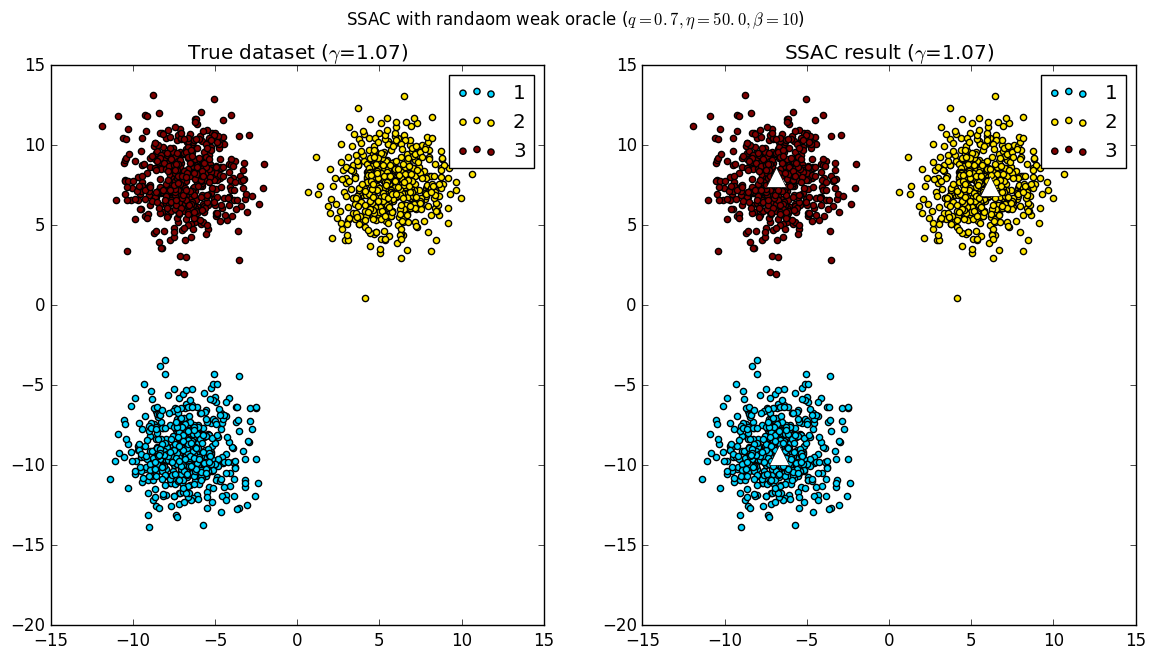}}
		\caption{Clustering results of SSAC algorithm with a $q$ random-weak oracle $(q=0.7)$. Five subfigures correspond to different $\eta$ values from $2$ to $50$ in order. In each subfigure, \textit{left} figure shows a ground truth dataset ($\gamma=1.07$), and \textit{right} figure shows the recovered clustering. White triangles represent cluster centers estimated from samples in Phase 1 of Algorithm \ref{alg:weak_SSAC_ran}.}
		\label{fig:res_q07}
	\end{center}
	\vskip -0.1in
\end{figure}
\newpage
\begin{figure}[ht]
	\begin{center}
		\centerline{\includegraphics[width=.48\linewidth]{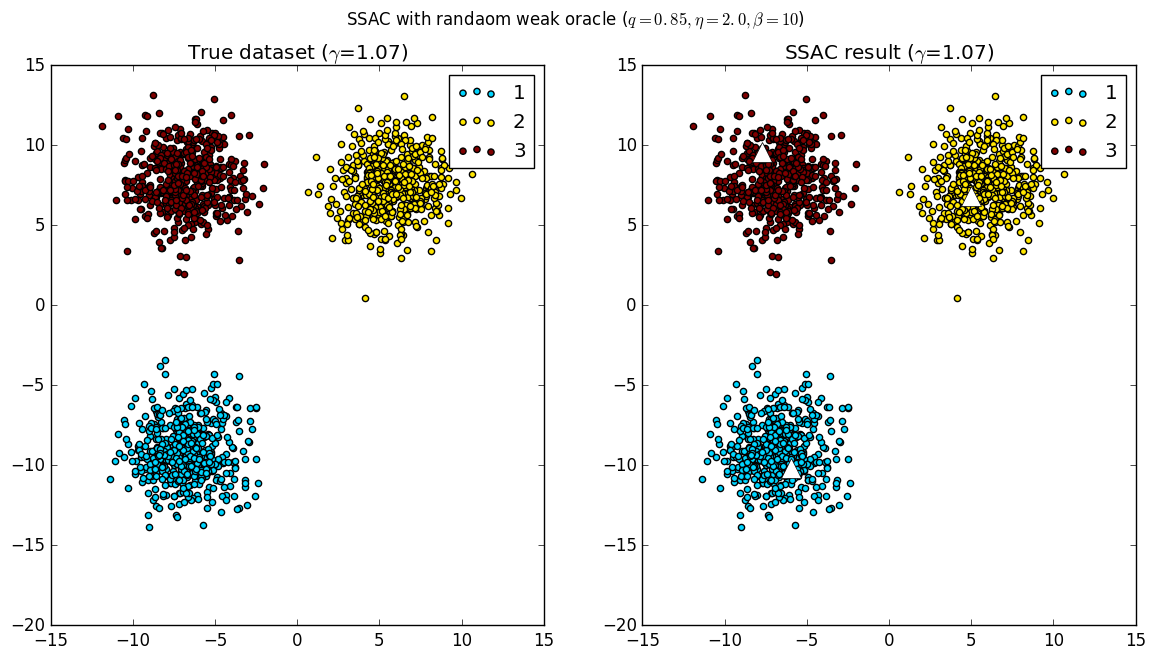}\hspace{1em}\includegraphics[width=.48\linewidth]{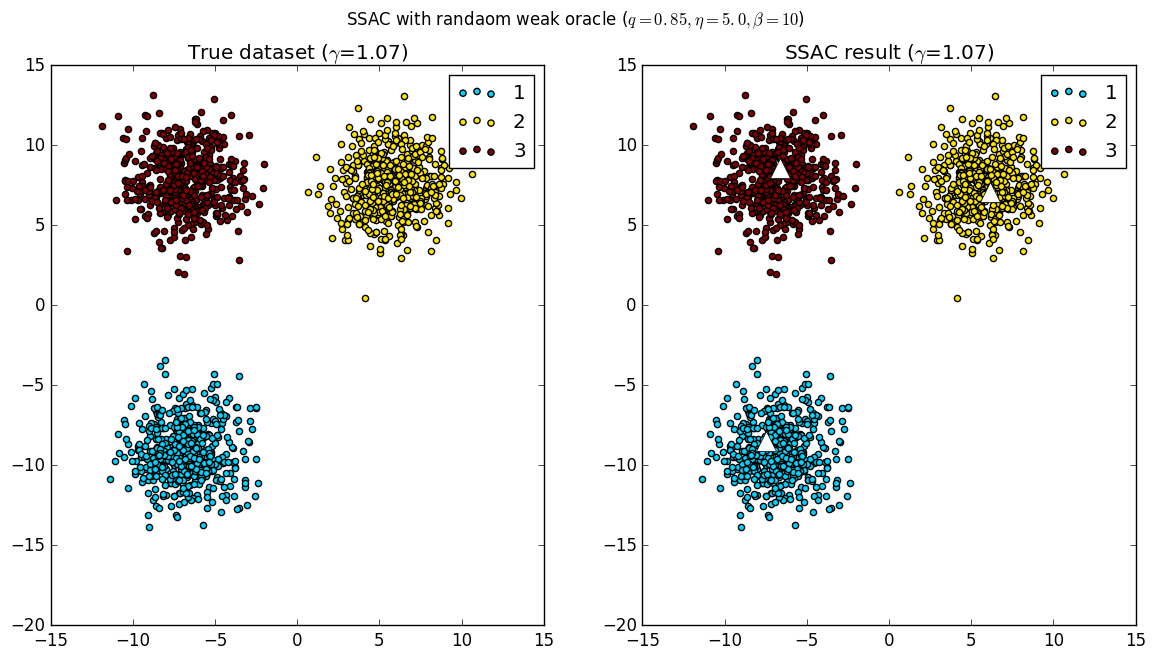}}
		\vspace{2em}
		\centerline{\includegraphics[width=.48\linewidth]{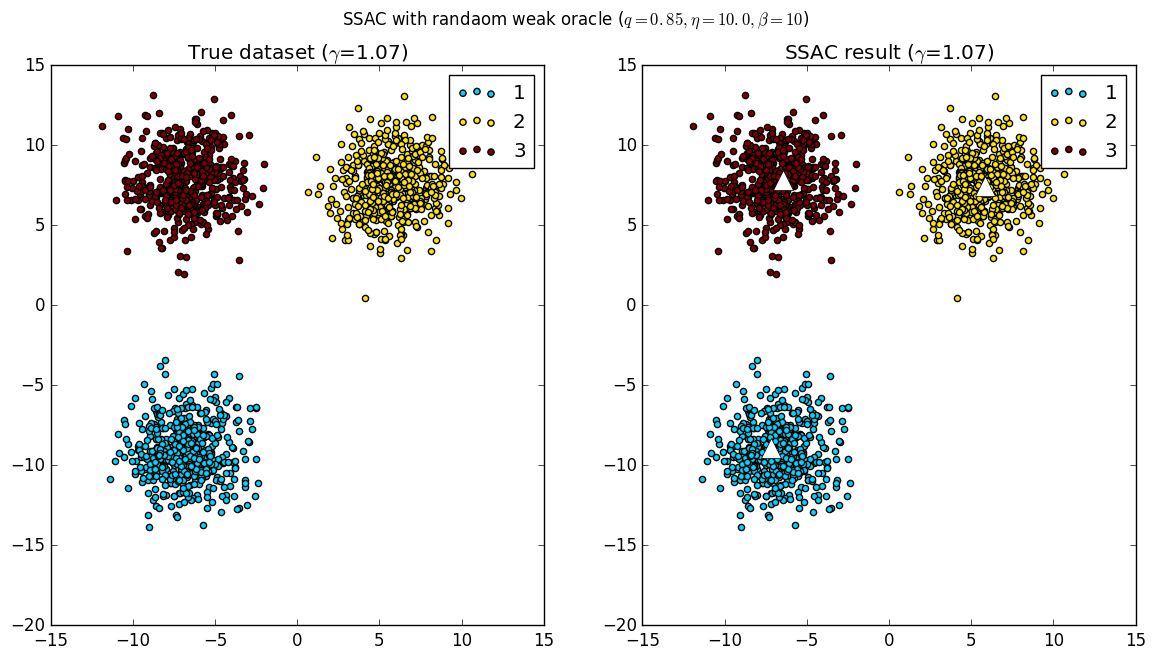}\hspace{1em}\includegraphics[width=.48\linewidth]{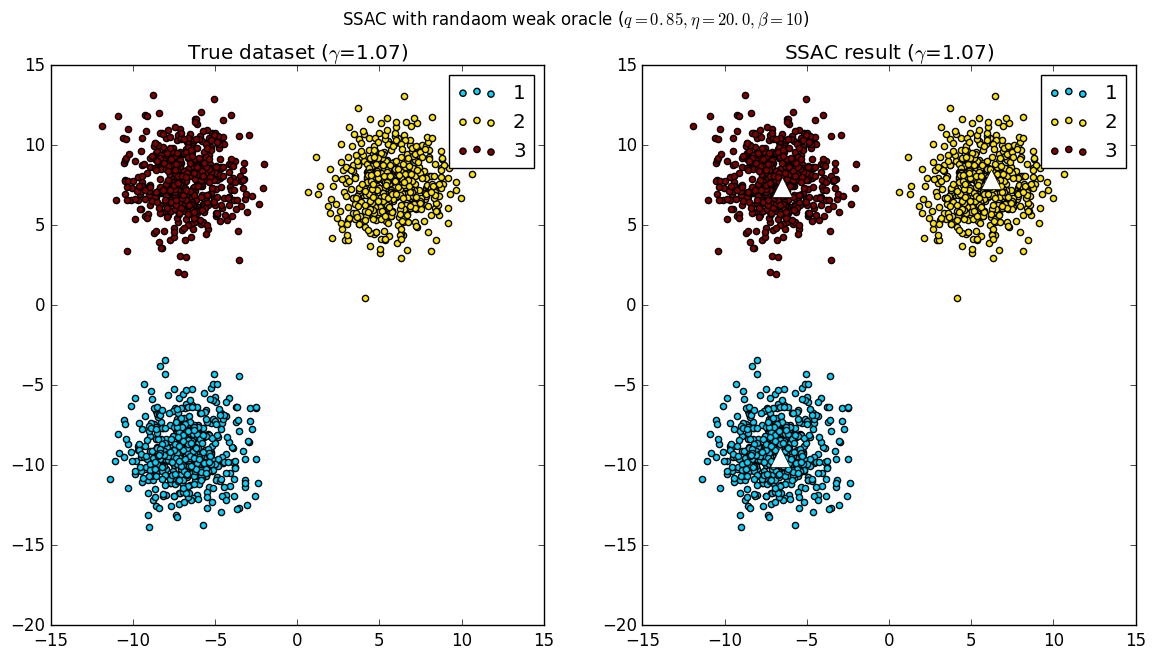}}
		\vspace{2em}
		\centerline{\includegraphics[width=.48\linewidth]{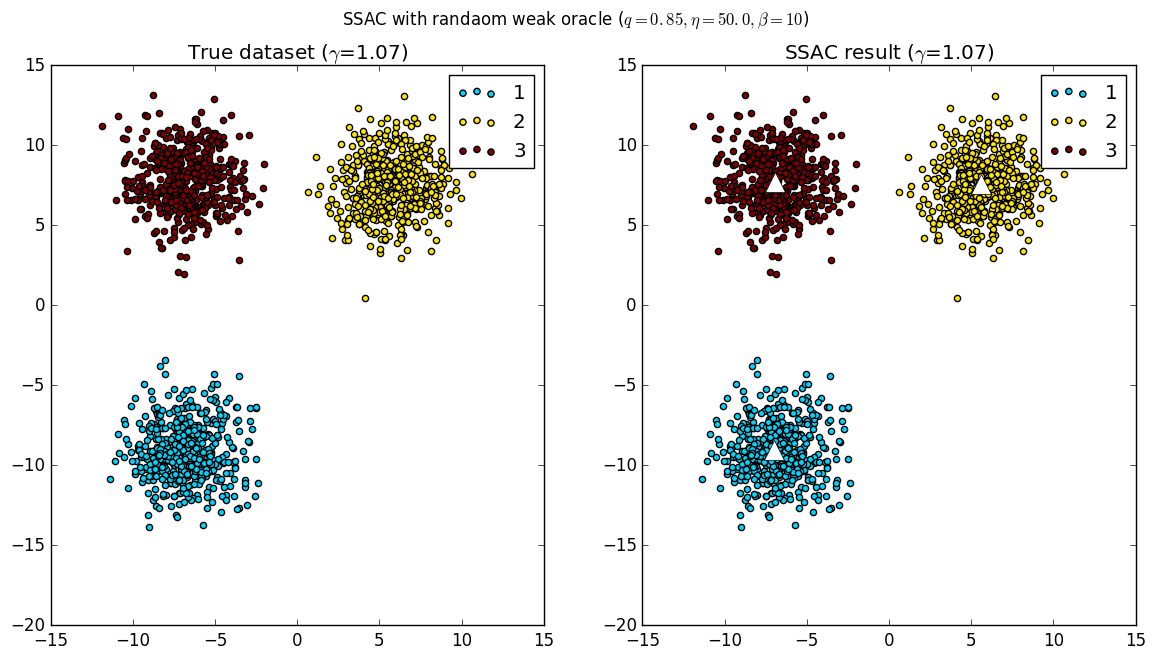}}
		\caption{Clustering results of SSAC algorithm with a $q$ random-weak oracle $(q=0.85)$. Five subfigures correspond to different $\eta$ values from $2$ to $50$ in order. In each subfigure, \textit{left} figure shows a ground truth dataset ($\gamma=1.07$), and \textit{right} figure shows the recovered clustering. White triangles represent cluster centers estimated from samples in Phase 1 of Algorithm \ref{alg:weak_SSAC_ran}.}
		\label{fig:res_q85}
	\end{center}
	\vskip -0.1in
\end{figure}
\newpage
\begin{figure}[ht]
	\begin{center}
		\centerline{\includegraphics[width=.48\linewidth]{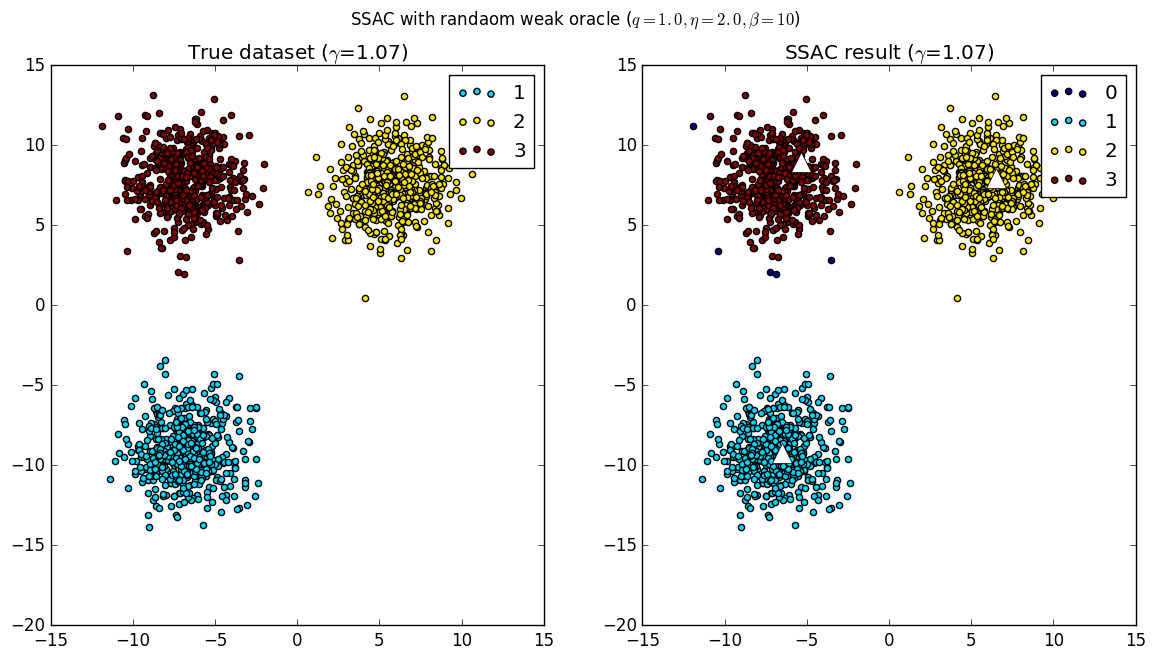}\hspace{1em}\includegraphics[width=.48\linewidth]{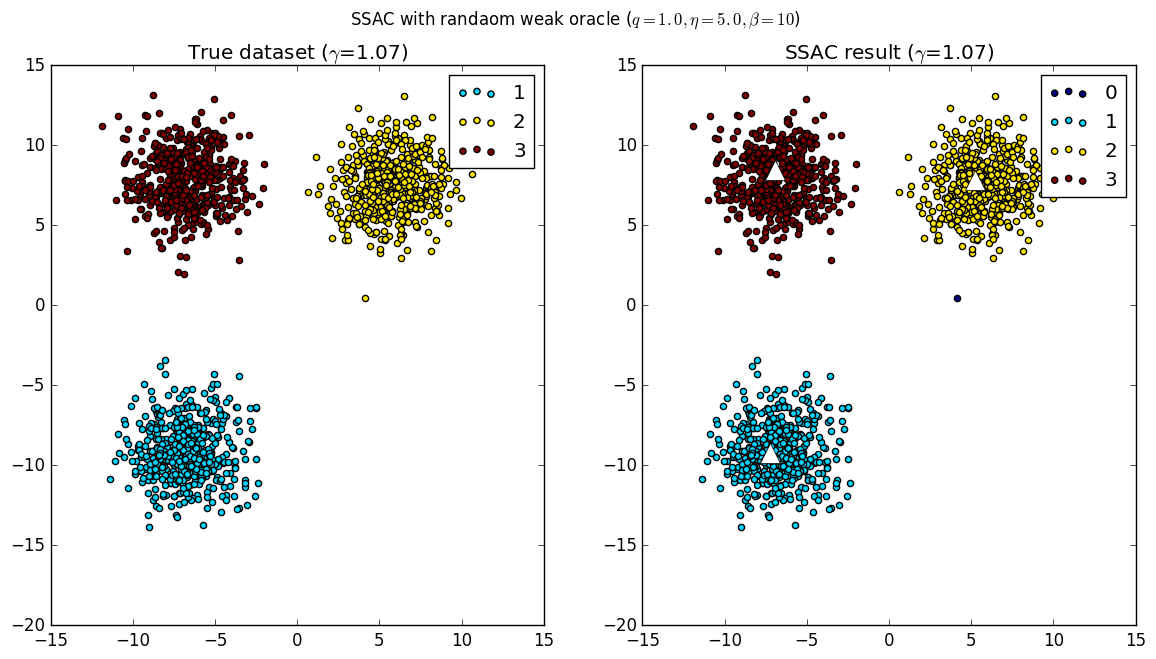}}
		\vspace{2em}
		\centerline{\includegraphics[width=.48\linewidth]{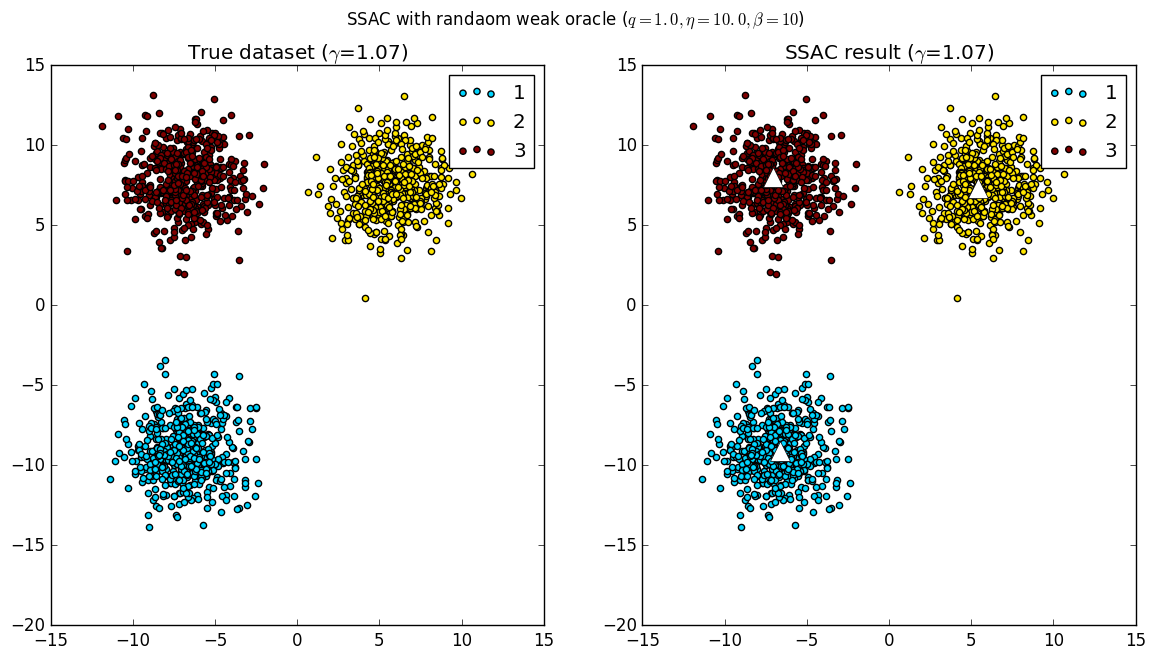}\hspace{1em}\includegraphics[width=.48\linewidth]{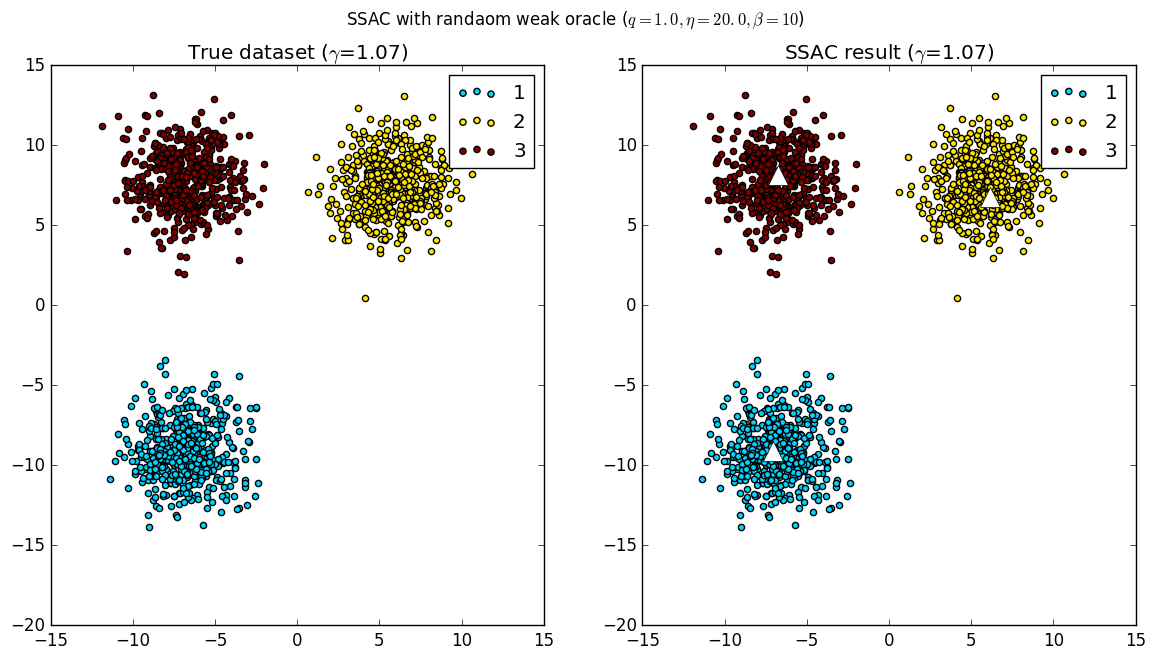}}
		\vspace{2em}
		\centerline{\includegraphics[width=.48\linewidth]{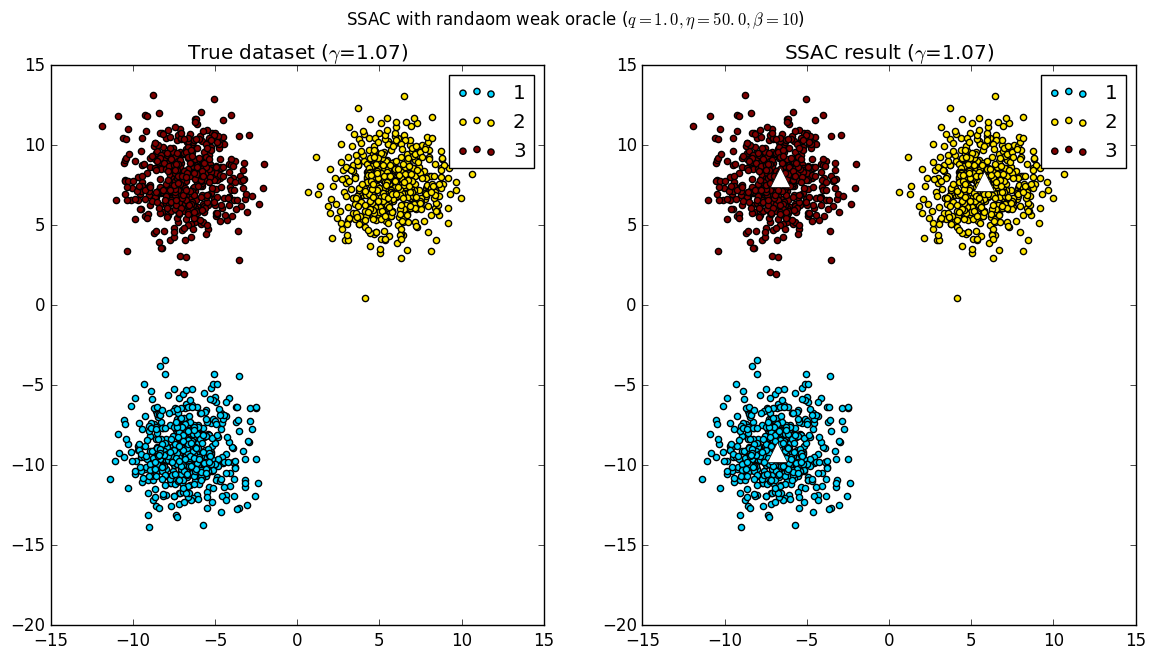}}
		\caption{Clustering results of SSAC algorithm with a $q$ random-weak oracle $(q=1.0)$. Five subfigures correspond to different $\eta$ values from $2$ to $50$ in order. In each subfigure, \textit{left} figure shows a ground truth dataset ($\gamma=1.07$), and \textit{right} figure shows the recovered clustering. White triangles represent cluster centers estimated from samples in Phase 1 of Algorithm \ref{alg:weak_SSAC_ran}.}
		\label{fig:res_q10}
	\end{center}
	\vskip -0.1in
\end{figure}

\newpage
\subsection{Additional Results}\label{subsec:append_results}
\paragraph{Higher Dimension} $N_{rep}=5000$, $n=3000$, $m=10$, $k=3$, $\sigma_{std}=1.75$, $\gamma_{\min}=1.0$, $\gamma_{\max}=1.1$, $\beta=10$.

First, we tested a case where data is 10-dimensional. Table \ref{table:append_results_highm} shows $Accuracy$ in percentage, and $\#Failure$ on different parameter pairs $(q,\eta)$ respectively. Also, Figure \ref{fig:append_res_acc_highm} visualizes the $Accuracy$ result as a graph.
\begin{figure}[ht]
	\centering
	\begin{subfigure}{.4\linewidth}
		\centering
		\includegraphics[width=\linewidth]{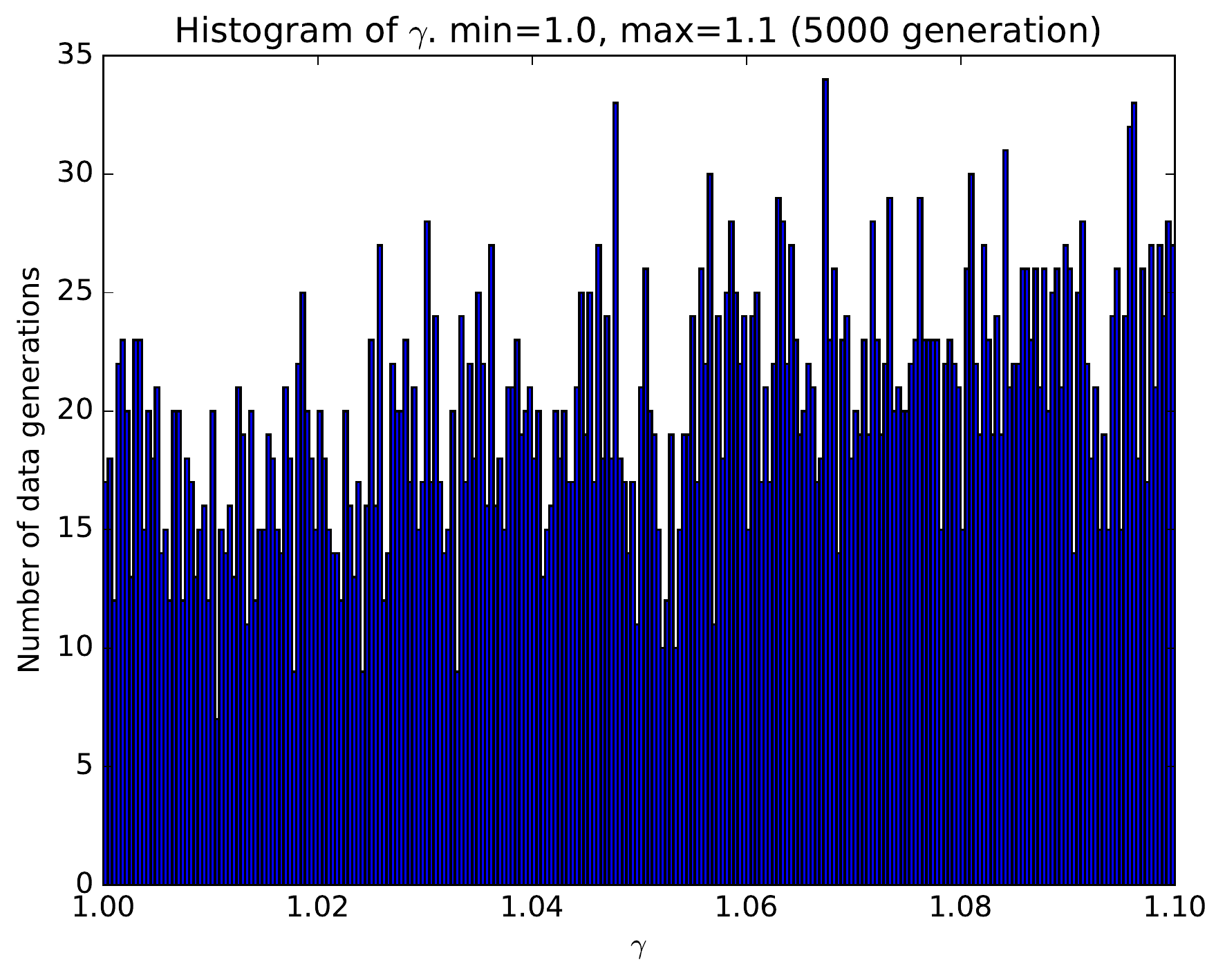}
		\caption{}\label{fig:append_gamma_hist_highm}
	\end{subfigure}
	\hspace {2em}
	\begin{subfigure}{.4\linewidth}
		\includegraphics[width=\linewidth]{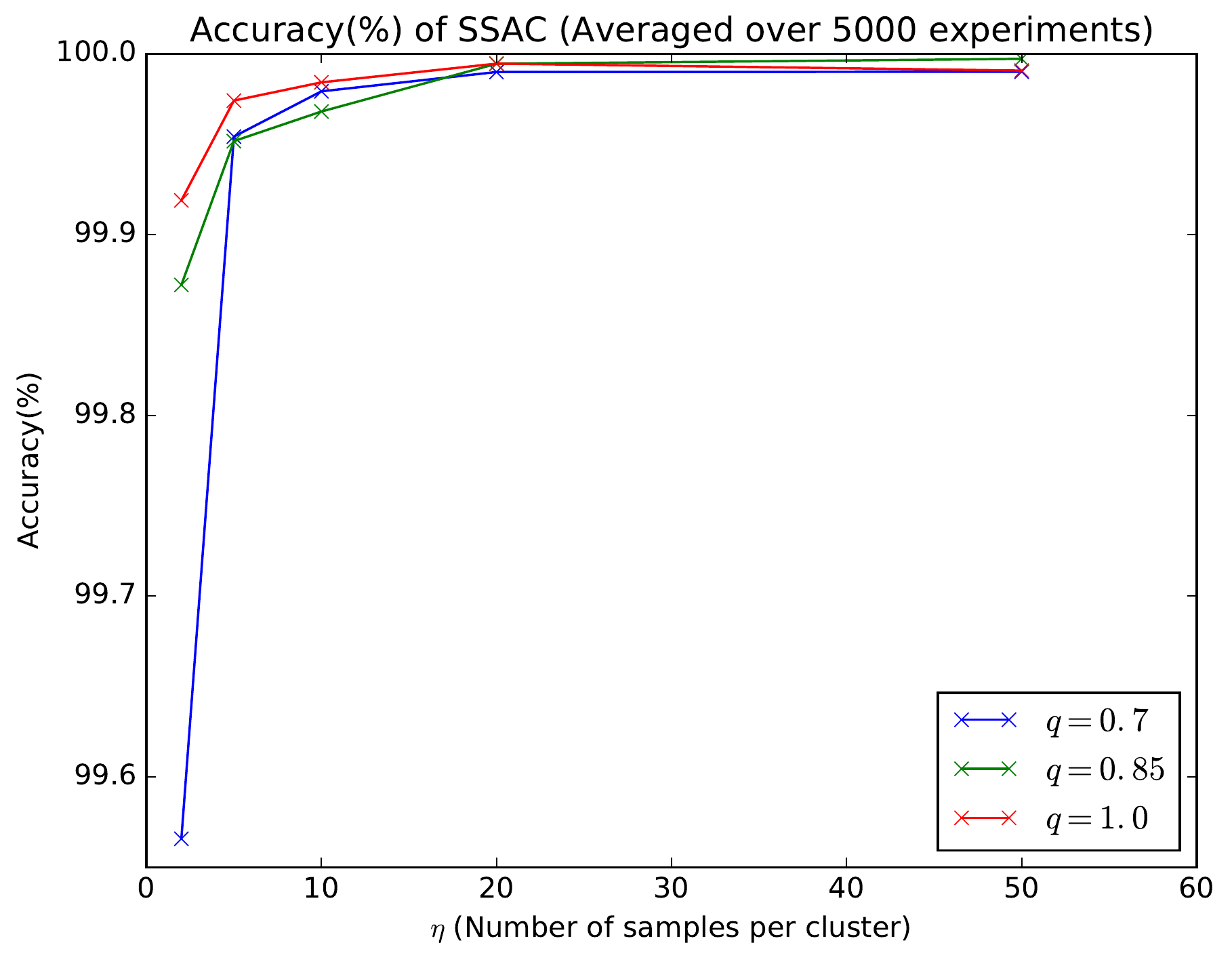}
		\caption{}\label{fig:append_res_acc_highm}
	\end{subfigure}
	\caption{$\gamma_{\min}=1.0$, $\gamma_{\max}=1.1$, $m=10$. (\subref{fig:append_gamma_hist_highm}) Histogram of $\gamma$ margins ($N_{rep}=5000$). \textit{x-axis}: $\gamma$ (Margin value), \textit{y-axis}: Normalized number of data generations corresponding to each $\gamma$. (\subref{fig:append_res_acc_highm}) $Accuracy~(\%)$ of SSAC algorithm. Averaged over $N_{rep}=1000$ experiments. \textit{x-axis}: $\eta$ (Number of samples) \textit{y-axis}: $Accuracy~(\%)$.}
	\vskip -0.1in
\end{figure} 
\begin{table}[ht]
	\begin{small}
		\centering
		\caption{$\gamma_{\min}=1.0$, $\gamma_{\max}=1.1$, $m=10$. (\textit{Left}) $Accuracy~(\%)$ of SSAC algorithm. Averaged over $N_{rep}=5000$ experiments. (\textit{Right}) $\# Failure$ of SSAC algorithm. Total sum over $N_{rep}=5000$ experiments.}
		\label{table:append_results_highm}
		\vspace{1em}
		\begin{subtable}{.58\linewidth}
			\centering
			\begin{tabular}{|c|ccccc|}
				\hline
				\multirow{2}{*}{$q$}& \multicolumn{5}{c|}{$\eta$} \\
				& 2 & 5 & 10 & 20 & 50 \\
				\hline
				\hline
				0.70 & 99.566 & 99.954 & 99.979 & 99.990 & 99.990 \\
				0.85 & 99.872 & 99.952 & 99.968 & 99.994 & 99.997 \\
				1.00 & 99.919 & 99.974 & 99.984 & 99.995 & 99.991 \\
				\hline
			\end{tabular}
		\end{subtable}
		\begin{subtable}{.38\linewidth}
			\centering
			\begin{tabular}{|c|ccccc|}
				\hline
				\multirow{2}{*}{$q$}& \multicolumn{5}{c|}{$\eta$} \\
				& 2 & 5 & 10 & 20 & 50 \\
				\hline
				\hline
				0.70 & 49 & 2 & 1 & 0 & 1 \\
				0.85 &  7 & 4 & 3 & 0 & 0 \\
				1.00 &  3 & 1 & 1 & 0 & 1 \\
				\hline
			\end{tabular}
		\end{subtable}
	\end{small}
\end{table}
\newpage
\paragraph{Non-separable} $N_{rep}=5000$, $n=1500$, $m=2$, $k=3$, $\sigma_{std}=1.75$, $\gamma_{\min}=0.6$, $\gamma_{\max}=1.0$.

Although our theoretical results assume $\gamma>1$ to have a clusterability, we tested our method cases where clusters overlap. Table \ref{table:append_results_nonsep} shows $Accuracy$ in percentage, and $\#Failure$ on different parameter pairs $(q,\eta)$ respectively. Also, Figure \ref{fig:append_res_acc_nonsep} visualizes the $Accuracy$ result as a graph. As expected, $Accuracy$ of the algorithm has decreased which is affected by points overlapping at the edge of clusters. Also, number of failures has increased compared to the ideal cases with $\gamma>1$. However, our result still shows that enough number of queries can give reasonable clustering if small portion of points overlap. Figure \ref{fig:append_res_q07_nonsep}, \ref{fig:append_res_q85_nonsep}, and \ref{fig:append_res_q10_nonsep} are also provided to show the visualization of clustering results.

\begin{figure}[ht]
	\centering
	\begin{subfigure}{.4\linewidth}
		\centering
		\includegraphics[width=\linewidth]{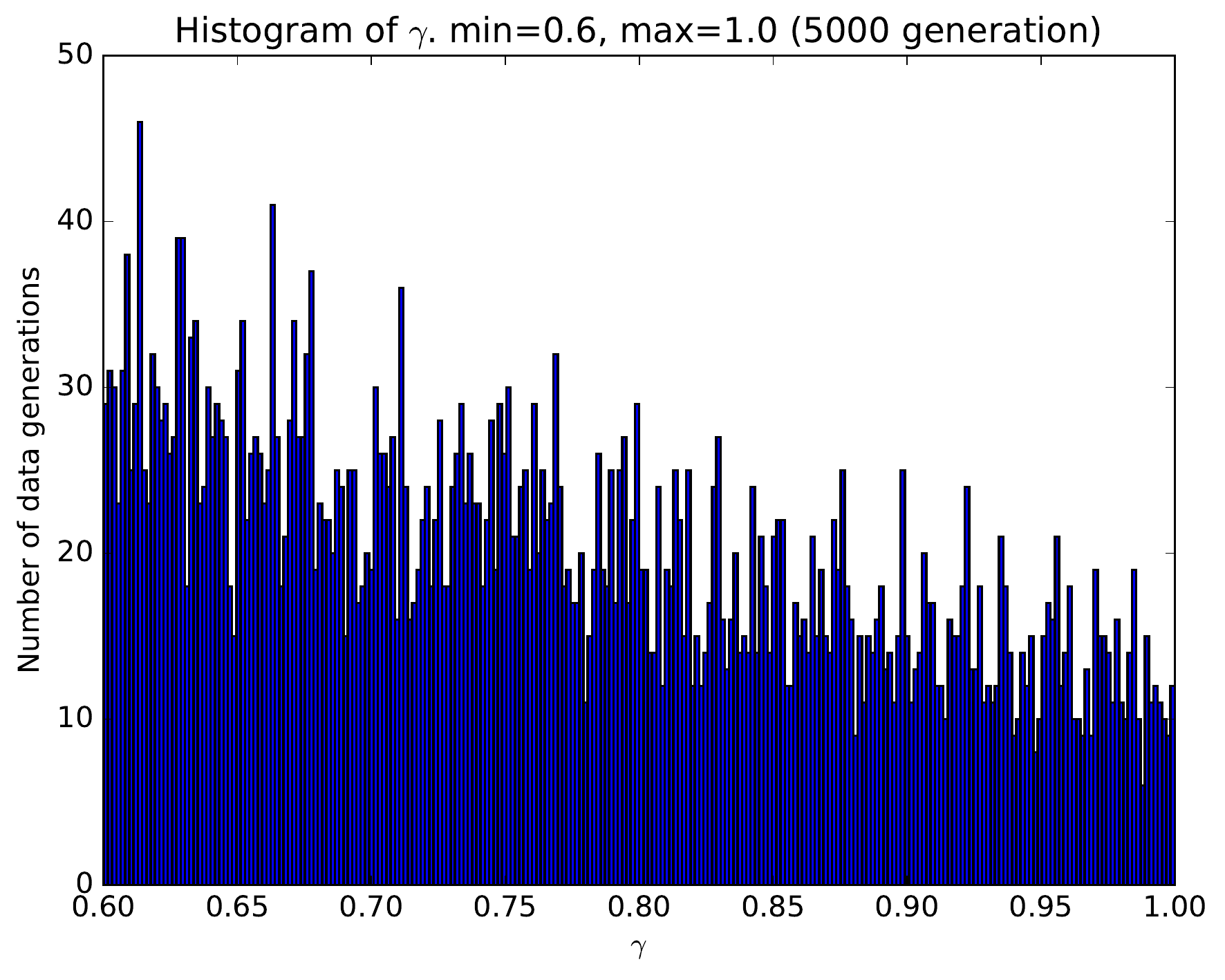}
		\caption{}\label{fig:append_gamma_hist_nonsep}
	\end{subfigure}
	\hspace {2em}
	\begin{subfigure}{.4\linewidth}
		\includegraphics[width=\linewidth]{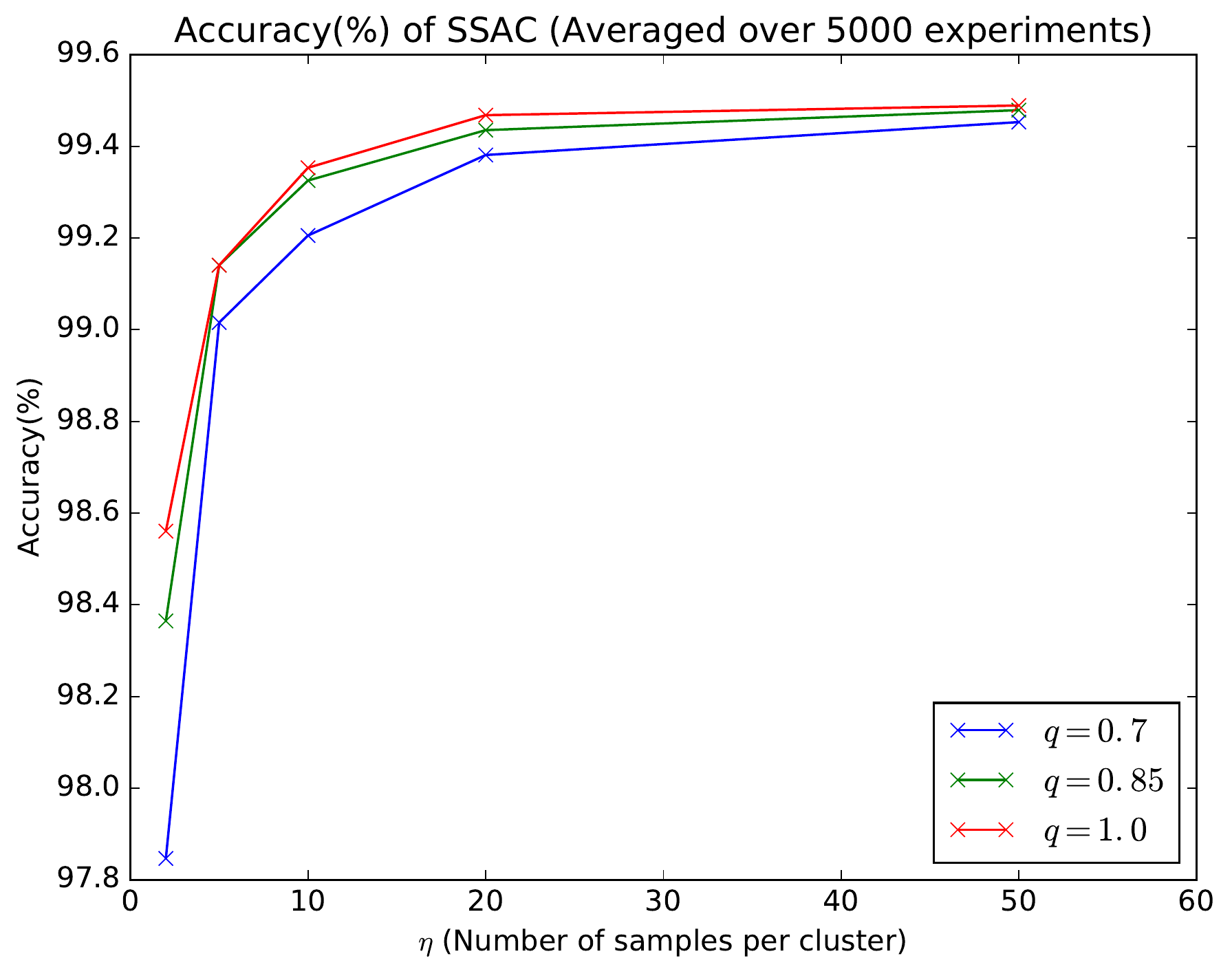}
		\caption{}\label{fig:append_res_acc_nonsep}
	\end{subfigure}
	\caption{$\gamma_{\min}=0.6$, $\gamma_{\max}=1.0$, $m=2$. (\subref{fig:append_gamma_hist_nonsep}) Histogram of $\gamma$ margins ($N_{rep}=5000$). \textit{x-axis}: $\gamma$ (Margin value), \textit{y-axis}: Normalized number of data generations corresponding to each $\gamma$. (\subref{fig:append_res_acc_nonsep}) $Accuracy~(\%)$ of SSAC algorithm. Averaged over $N_{rep}=1000$ experiments. \textit{x-axis}: $\eta$ (Number of samples) \textit{y-axis}: $Accuracy~(\%)$.}
	\vskip -0.1in
\end{figure} 
\begin{table}[ht]
	\begin{small}
		\centering
		\caption{$\gamma_{\min}=0.6$, $\gamma_{\max}=1.0$, $m=2$. (\textit{Left}) $Accuracy~(\%)$ of SSAC algorithm. Averaged over $N_{rep}=5000$ experiments. (\textit{Right}) $\# Failure$ of SSAC algorithm. Total sum over $N_{rep}=5000$ experiments.}
		\label{table:append_results_nonsep}
		\vspace{1em}
		\begin{subtable}{.58\linewidth}
			\centering
			\begin{tabular}{|c|ccccc|}
				\hline
				\multirow{2}{*}{$q$}& \multicolumn{5}{c|}{$\eta$} \\
				& 2 & 5 & 10 & 20 & 50 \\
				\hline
				\hline
				0.70 & 97.847 & 99.016 & 99.205 & 99.381 & 99.453 \\
				0.85 & 98.365 & 99.141 & 99.325 & 99.435 & 99.479 \\
				1.00 & 98.560 & 99.141 & 99.353 & 99.468 & 99.489 \\
				\hline
			\end{tabular}
		\end{subtable}
		\begin{subtable}{.38\linewidth}
			\centering
			\begin{tabular}{|c|ccccc|}
				\hline
				\multirow{2}{*}{$q$}& \multicolumn{5}{c|}{$\eta$} \\
				& 2 & 5 & 10 & 20 & 50 \\
				\hline
				\hline
				0.70 & 137 & 38 & 37 & 21 & 18 \\
				0.85 & 85 & 31 & 23 & 15 & 16 \\
				1.00 & 72 & 37 & 21 & 13 & 14 \\
				\hline
			\end{tabular}
		\end{subtable}
	\end{small}
\end{table}

\newpage
\begin{figure}[ht]
	\begin{center}
		\centerline{\includegraphics[width=.48\linewidth]{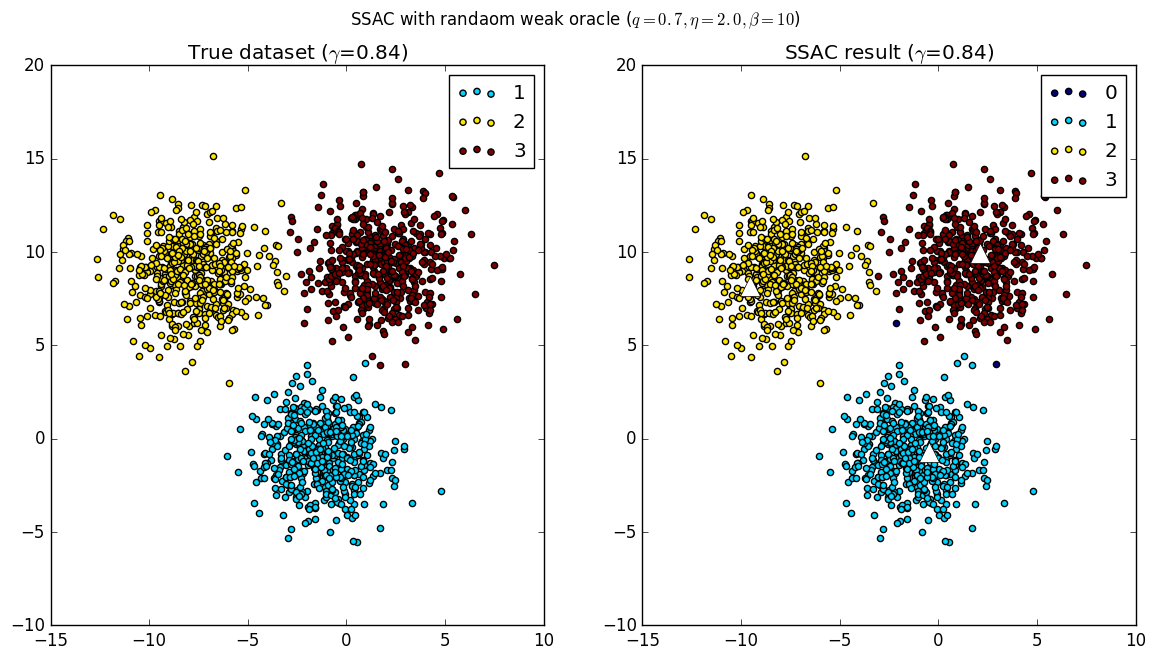}\hspace{1em}\includegraphics[width=.48\linewidth]{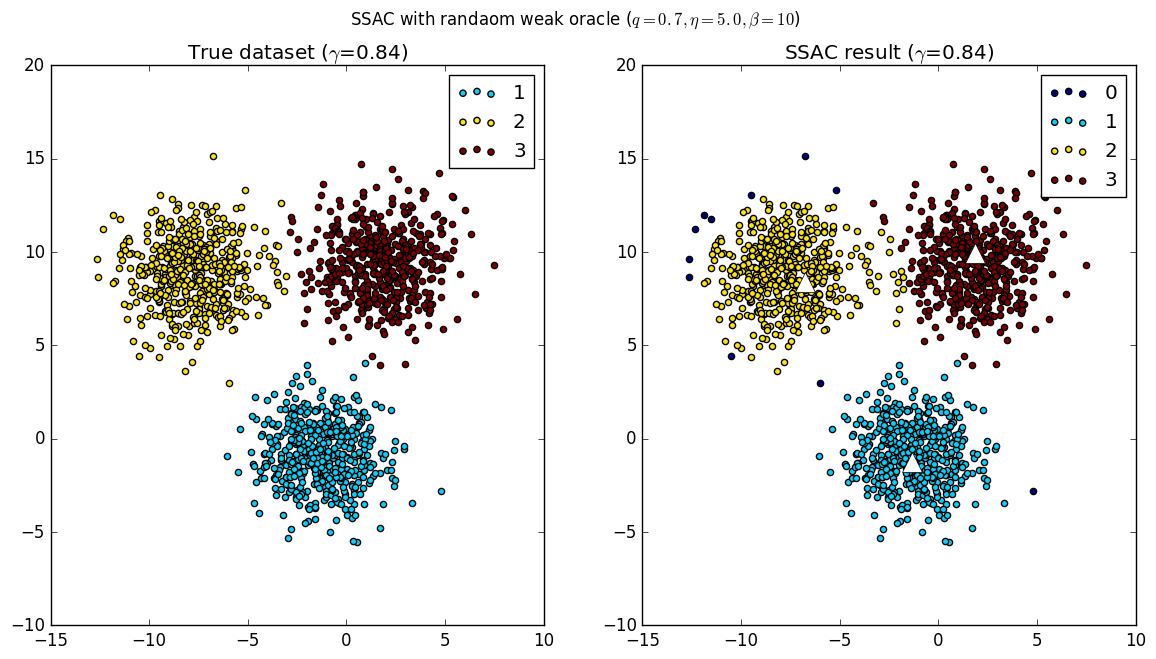}}
		\vspace{2em}
		\centerline{\includegraphics[width=.48\linewidth]{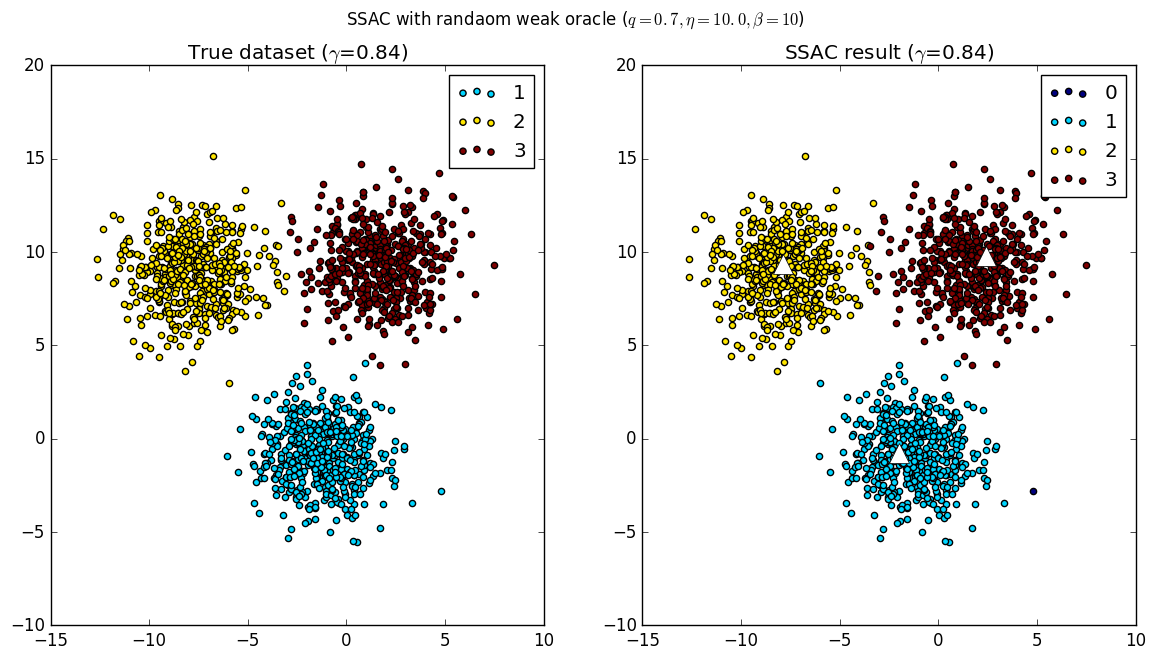}\hspace{1em}\includegraphics[width=.48\linewidth]{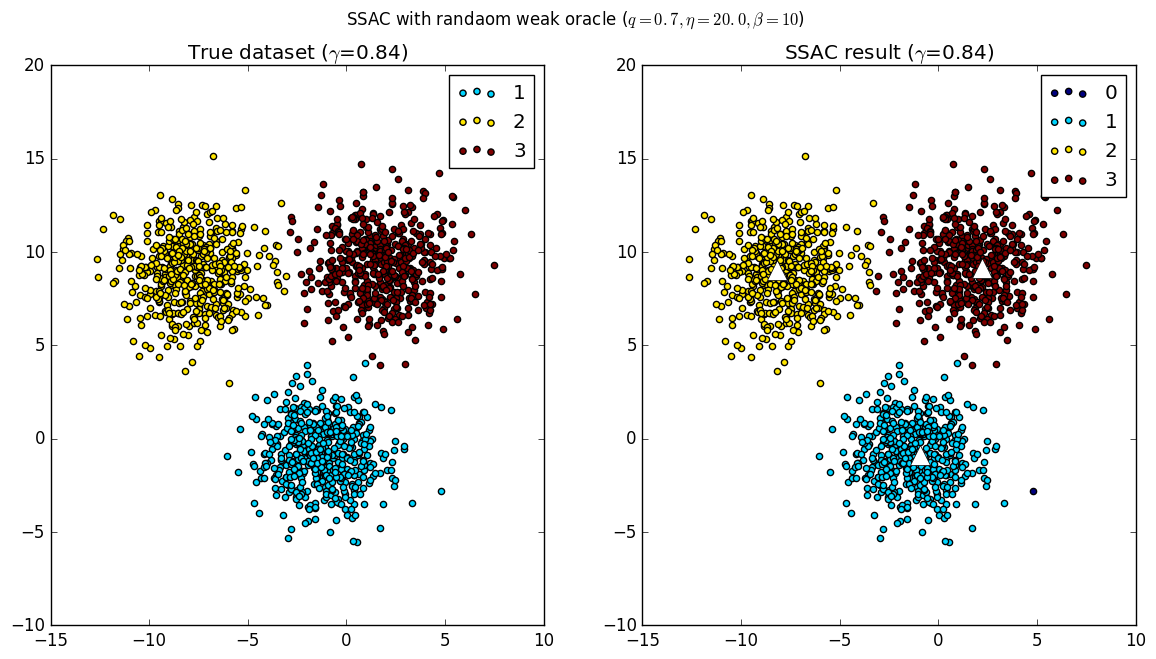}}
		\vspace{2em}
		\centerline{\includegraphics[width=.48\linewidth]{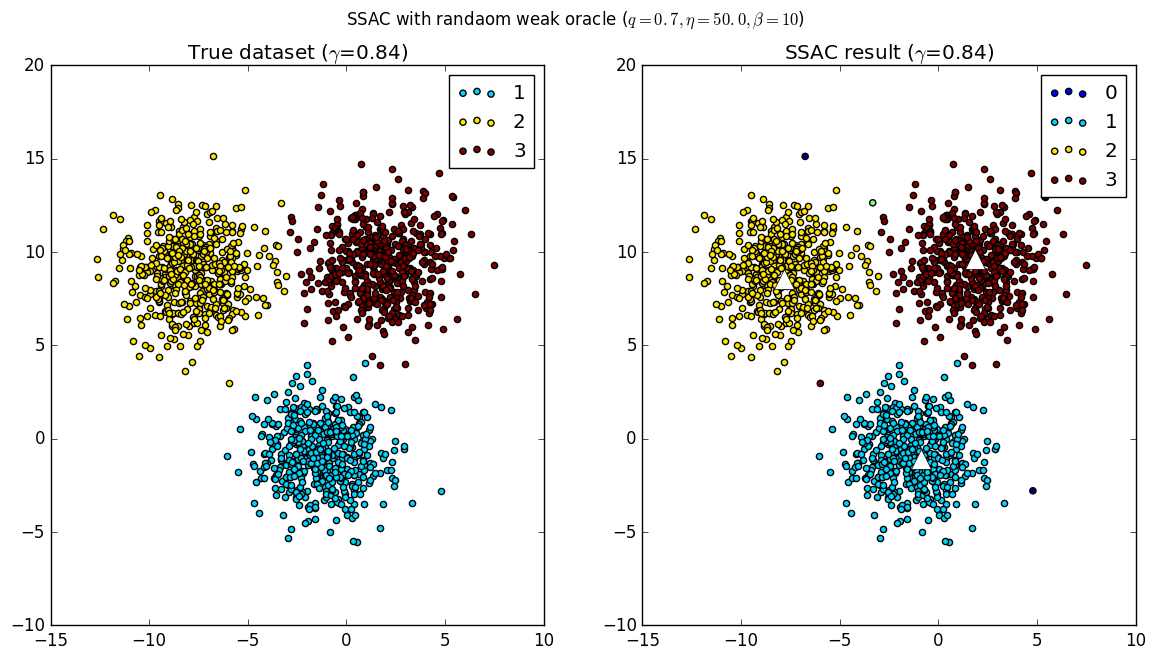}}
		\caption{Clustering results of SSAC algorithm with a $q$ random-weak oracle $(q=0.7)$. Five subfigures correspond to different $\eta$ values from $2$ to $50$ in order. In each subfigure, \textit{left} figure shows a ground truth dataset ($\gamma=0.84$), and \textit{right} figure shows the recovered clustering. White triangles represent cluster centers estimated from samples in Phase 1 of Algorithm \ref{alg:weak_SSAC_ran}.}
		\label{fig:append_res_q07_nonsep}
	\end{center}
	\vskip -0.1in
\end{figure}
\newpage
\begin{figure}[ht]
	\begin{center}
		\centerline{\includegraphics[width=.48\linewidth]{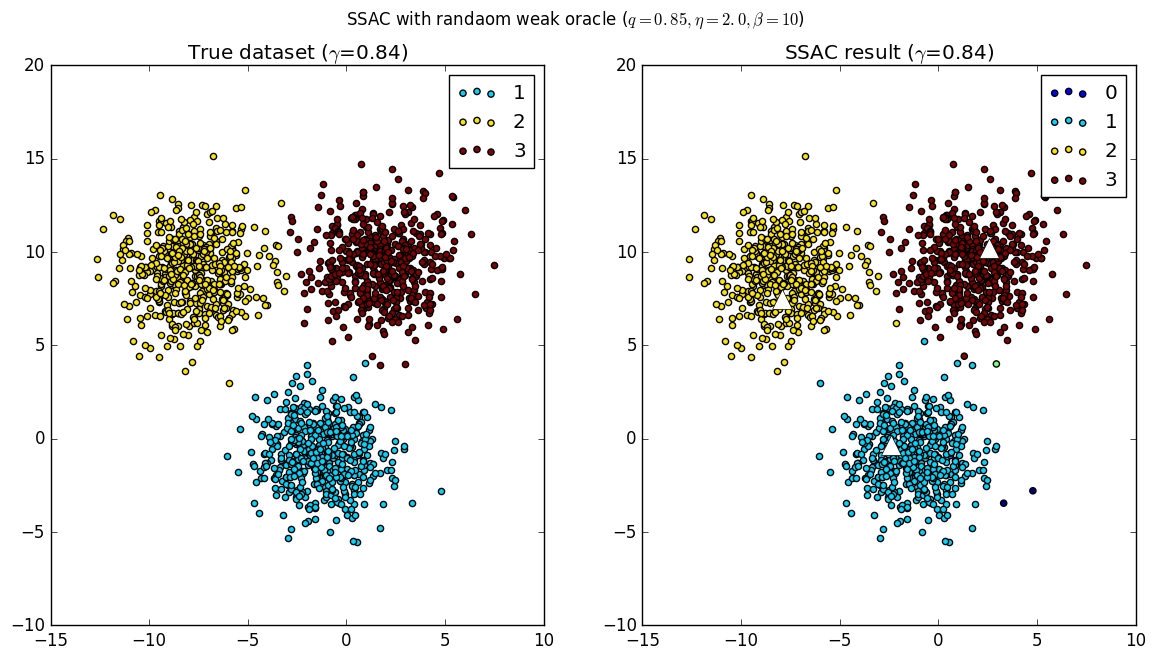}\hspace{1em}\includegraphics[width=.48\linewidth]{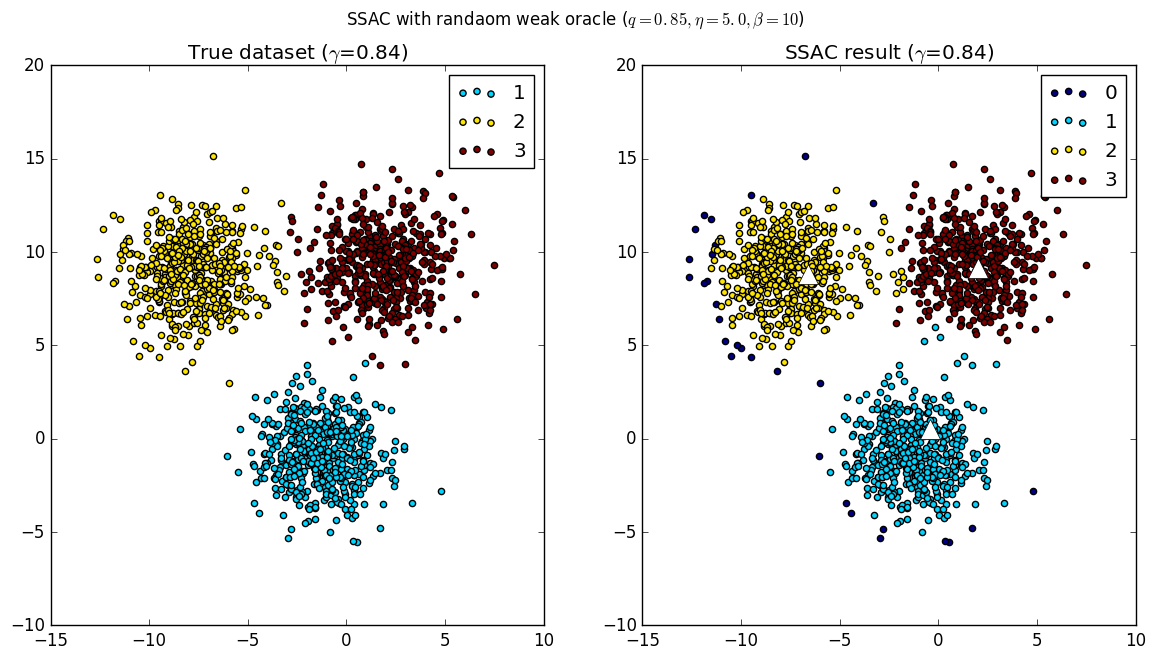}}
		\vspace{2em}
		\centerline{\includegraphics[width=.48\linewidth]{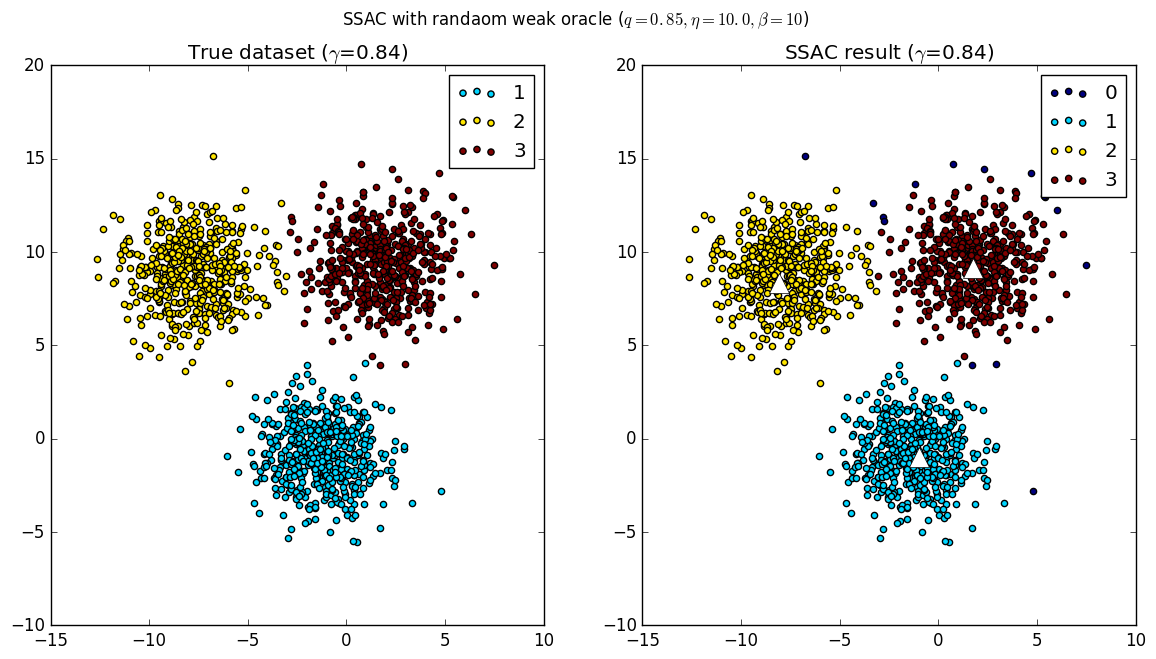}\hspace{1em}\includegraphics[width=.48\linewidth]{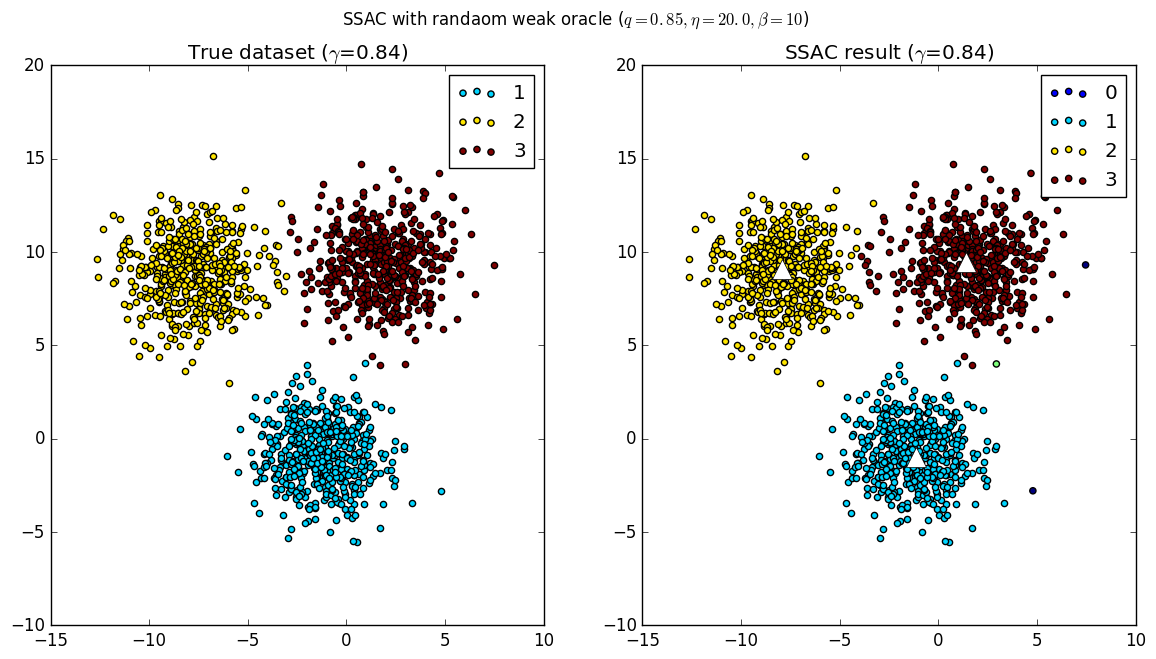}}
		\vspace{2em}
		\centerline{\includegraphics[width=.48\linewidth]{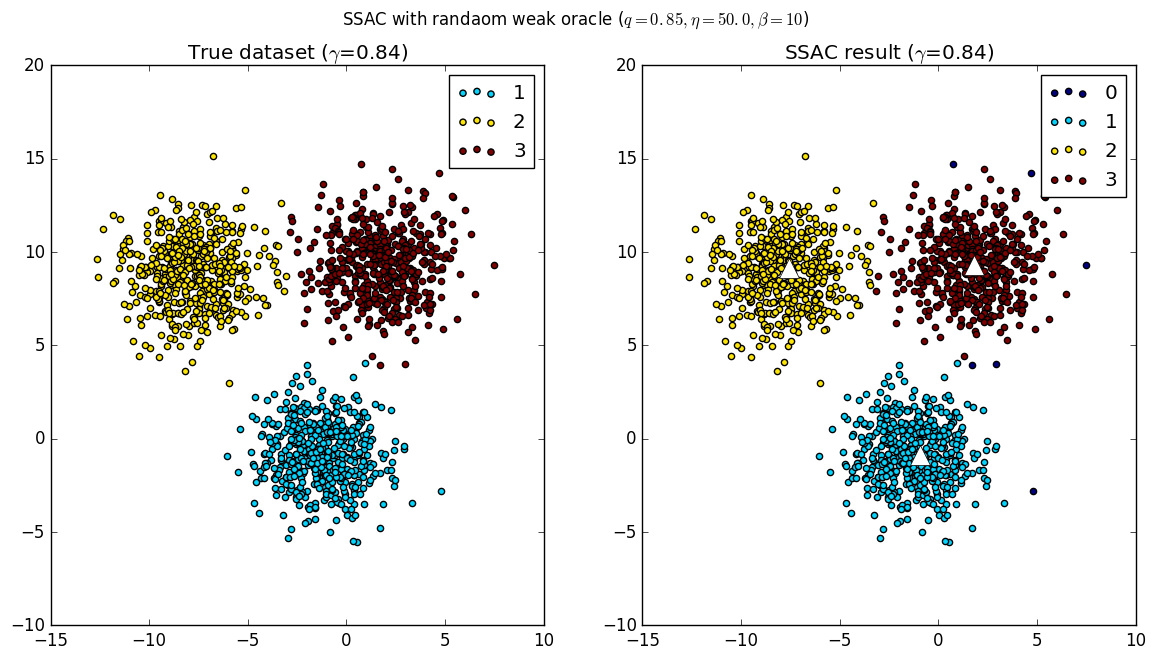}}
		\caption{Clustering results of SSAC algorithm with a $q$ random-weak oracle $(q=0.85)$. Five subfigures correspond to different $\eta$ values from $2$ to $50$ in order. In each subfigure, \textit{left} figure shows a ground truth dataset ($\gamma=0.84$), and \textit{right} figure shows the recovered clustering. White triangles represent cluster centers estimated from samples in Phase 1 of Algorithm \ref{alg:weak_SSAC_ran}.}
		\label{fig:append_res_q85_nonsep}
	\end{center}
	\vskip -0.1in
\end{figure}
\newpage
\begin{figure}[ht]
	\begin{center}
		\centerline{\includegraphics[width=.48\linewidth]{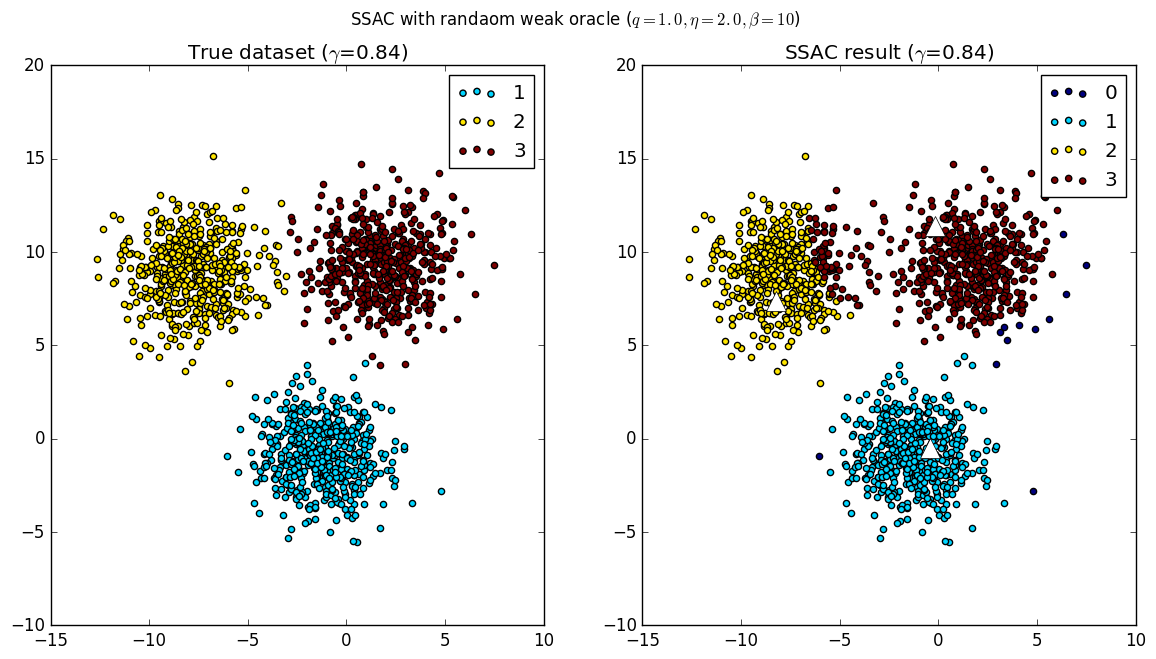}\hspace{1em}\includegraphics[width=.48\linewidth]{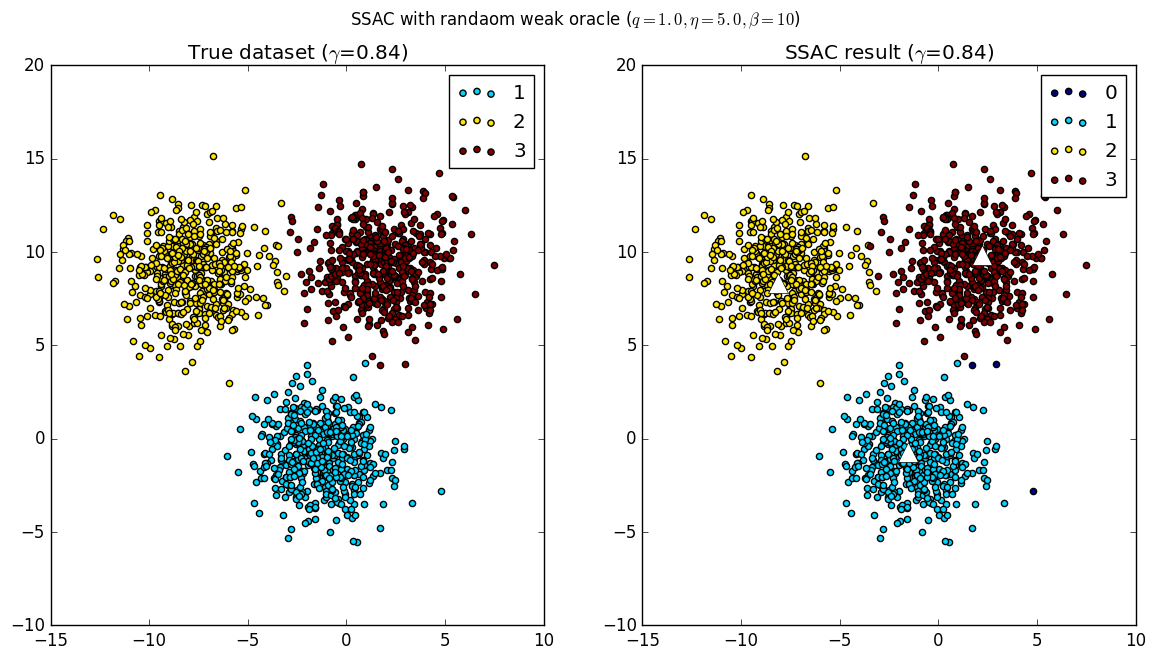}}
		\vspace{2em}
		\centerline{\includegraphics[width=.48\linewidth]{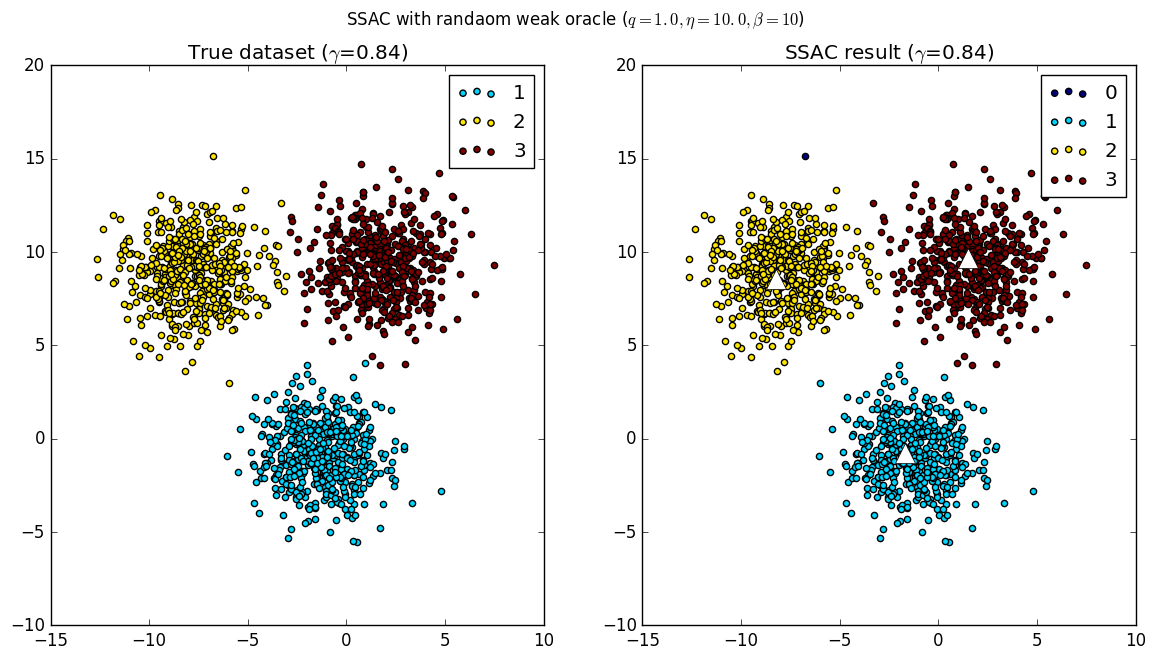}\hspace{1em}\includegraphics[width=.48\linewidth]{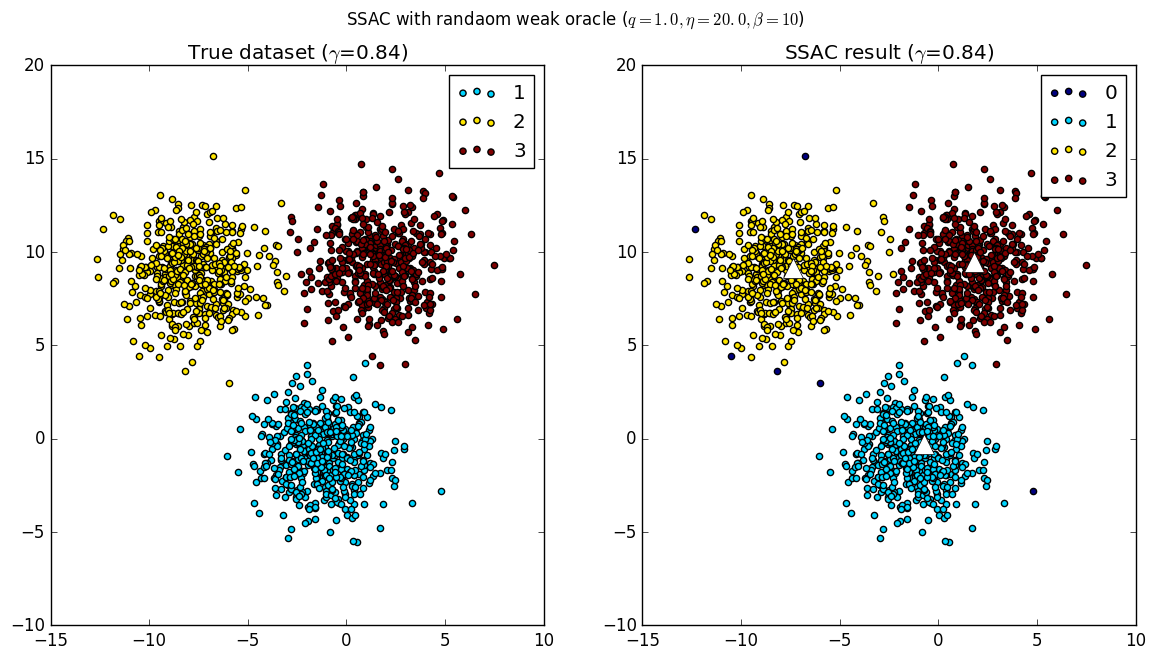}}
		\vspace{2em}
		\centerline{\includegraphics[width=.48\linewidth]{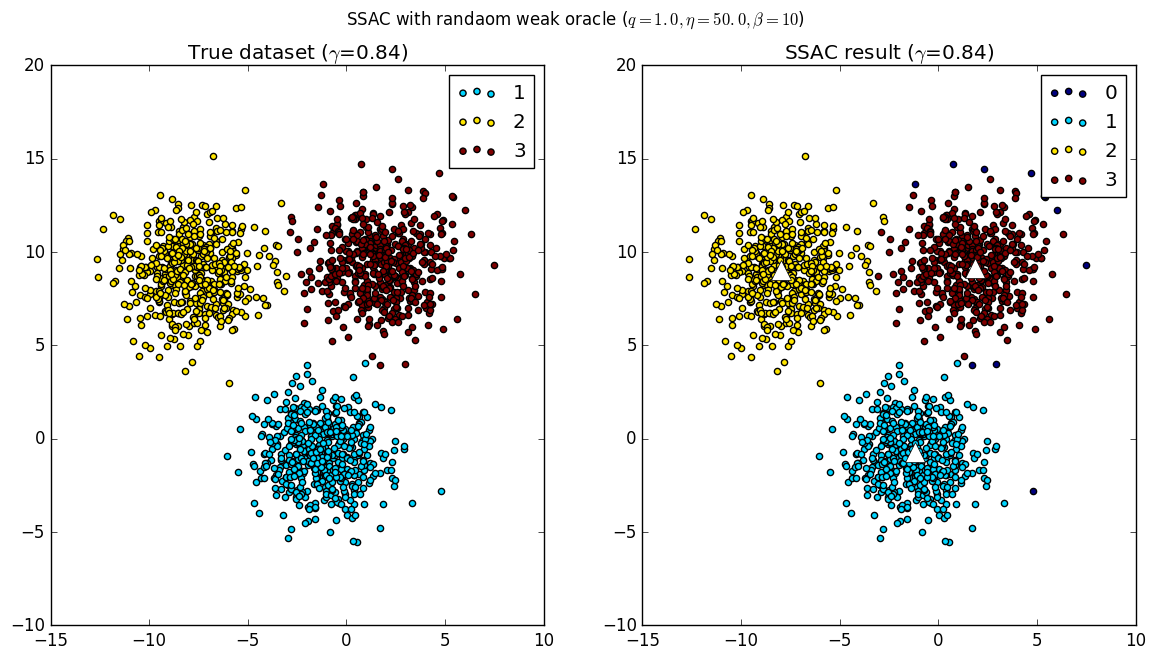}}
		\caption{Clustering results of SSAC algorithm with a $q$ random-weak oracle $(q=1.0)$. Five subfigures correspond to different $\eta$ values from $2$ to $50$ in order. In each subfigure, \textit{left} figure shows a ground truth dataset ($\gamma=0.84$), and \textit{right} figure shows the recovered clustering. White triangles represent cluster centers estimated from samples in Phase 1 of Algorithm \ref{alg:weak_SSAC_ran}.}
		\label{fig:append_res_q10_nonsep}
	\end{center}
	\vskip -0.1in
\end{figure}